\documentclass{article}
\pdfoutput=1

     \usepackage[nonatbib,preprint]{neurips_2020}

\usepackage[utf8]{inputenc} 
\usepackage[T1]{fontenc}    
\usepackage{hyperref}       
\usepackage{url}            
\usepackage{booktabs}       
\usepackage{amsfonts}       
\usepackage{nicefrac}       
\usepackage{microtype}      

\usepackage{latexsym,amsmath,amssymb,amsthm,eucal,bbm,color}
\usepackage{url}
\usepackage{comment}
\usepackage{float}
\usepackage{tcolorbox}
\usepackage{booktabs}
\usepackage{blindtext}
\usepackage{hyperref}
\hypersetup{
    colorlinks,
    linkcolor={red!50!black},
    citecolor={blue!50!black},
    urlcolor={blue!80!black}
}
\usepackage{multirow}
\usepackage{booktabs}
\usepackage{array}
\usepackage{soul}
 
\usepackage[normalem]{ulem}
\usepackage{stackengine}
\usepackage{parskip}
\usepackage{bm}
\usepackage{balance}
\usepackage{comment}
\usepackage{soul}
\usepackage{listings}
\usepackage[version=4]{mhchem}
\usepackage{empheq}
\usepackage{mdframed}
\usepackage{nomencl,etoolbox,ragged2e,siunitx}
\usepackage{tikz}
\usetikzlibrary{shapes,arrows}
\usetikzlibrary{er,positioning}

\usetikzlibrary{fit,positioning}
\tikzstyle{block} = [rectangle, draw, fill=blue!20, 
    text width=12.8em, text centered, rounded corners, minimum height=4em]
\tikzstyle{line} = [draw, -latex']
\tikzstyle{cloud} = [draw, ellipse,fill=red!20, node distance=3cm,
    minimum height=2em]
\usepackage{mdframed}

\newtheorem{thm}{Theorem}

\newtheorem{prop}{Proposition}
\newtheorem{cor}{Corollary}
\newtheorem{lemma}{Lemma}
\newtheorem{defn}{Definition}

\newtheorem*{remark*}{Remark}

\newmdtheoremenv{mhyp}{Hypothesis} 
\newmdtheoremenv{mthm}{Theorem}
\newmdtheoremenv{mtheorem}{Theorem}
\newmdtheoremenv{mprop}{Proposition}
\newmdtheoremenv{mcor}{Corollary}
\newmdtheoremenv{mlemma}{Lemma}
\newmdtheoremenv{mdefn}{Definition}
\newmdtheoremenv{mmydef}{Definition}
\newmdtheoremenv{mconj}{Conjecture}
\newmdtheoremenv{mex}{Example}
\newmdtheoremenv{mexercise}{Exercise}
\usepackage{wrapfig}
\usepackage{caption}

\usepackage{tikz}
\usetikzlibrary{calc}
\usepackage{pgfplots}
\usepackage{mwe}
\usepackage{cancel}
\DeclareMathAlphabet\mathbfcal{OMS}{cmsy}{b}{n}
\DeclareMathOperator*{\argmax}{arg\,max}
\DeclareMathOperator*{\argmin}{arg\,min}

\allowdisplaybreaks

\def \bx{\boldsymbol{x}}

\def \by{\boldsymbol{y}}
\def \bz{\boldsymbol{z}}
\def \bm{\boldsymbol{m}}
\def \bv{\boldsymbol{v}}

\def \bh{\boldsymbol{h}}

\def \bu{\boldsymbol{u}}

\def \bv{\boldsymbol{v}}
\def \bb{\boldsymbol{b}}

\def \bepsilon{\boldsymbol{\epsilon}}
\def \bq{\boldsymbol{q}}
\def \bm{\boldsymbol{m}}
\def \bw{\boldsymbol{w}}
\def \be{\boldsymbol{e}}
\def \br{\boldsymbol{r}}

\def \ba{\boldsymbol{a}}
\def \bA{\boldsymbol{A}}

\def \bV{\boldsymbol{V}}

\def \bW{\boldsymbol{W}}
\def \bt{\boldsymbol{t}}
\def \bl{\boldsymbol{l}}

\def \bE{\boldsymbol{E}}

\def \bQ{\boldsymbol{D}}
\def \bM{\boldsymbol{M}}

\def \bSigma{\boldsymbol{\Sigma}}

\def \bmu{\boldsymbol{\mu}}

\def\R{{\mathbb R}}

\def\Indic{\mathbbm{1}}

\newcommand{\im}{\text{Im}}
\newcommand{\card}{\text{Card}}
\newcommand{\trace}{\text{Tr}}
\newcommand{\vect}{\text{vect}}
\newcommand{\sign}{\text{sign}}

\newcommand{\diag}{\text{diag}}

\newcommand{\qq}{\vspace*{-2mm}}

\newcommand{\brewriteA}{
\bA^{\ell+1\rightarrow L}\bQ_{\omega}^{\ell}\bW^{\ell}\bQ_{\omega}^{\ell-1}\bb_{\omega}^{1\rightarrow \ell-1}
}

\newcommand{\brewriteb}{
\sum_{i=\ell}^{L}\bA_{\omega}^{i + 1 \rightarrow L}\bQ^{i}_{\omega}\bv^{i}
}

\newcommand{\Arewrite}
{
\bA^{\ell+1\rightarrow L}_{\omega}\bQ^{\ell}_{\omega}\bW^{\ell}\bQ^{\ell-1}_{\omega}\bA^{1\rightarrow \ell -1}
}

\newcommand{\loss}
{
-\frac{1}{2}\log\Big( (2\pi)^{S+D}|\det(\bSigma_{\bx})||\det(\bSigma_{\bz})| \Big)-\frac{1}{2}\Bigg(\bx^T\bSigma_{\bx}^{-1}\bx
    -2 \bx^T\bSigma_{\bx}^{-1}\left(\sum_{\omega}\bA_{\omega}\bm^1_{\omega}(\bx)+\bb_{\omega}m^0_{\omega}(\bx)\right)\\
    &
    +\sum_{\omega} m^0_{\omega}\bb_{\omega}^T\bSigma^{-1}_{\bx}\bb_{\omega}  +\trace(\bA_{\omega}^T\bSigma_{\bx}^{-1}\bA_{\omega}\bM^2_{\omega}(\bx))+2 (\bA_{\omega}\bm_{\omega}^1(\bx))^T\bSigma_{\bx}^{-1} \bb_{\omega}\Bigg)-\frac{1}{2}\trace(\bSigma_{\bz}^{-1}\bM^2(\bx))
}

\usepackage{xcolor}

\usepackage{tikz}
\usetikzlibrary{shapes,arrows}
\usetikzlibrary{er,positioning}

\usetikzlibrary{fit,positioning}
\tikzstyle{block} = [rectangle, draw, fill=blue!20, 
    text width=12.8em, text centered, rounded corners, minimum height=4em]
\tikzstyle{line} = [draw, -latex']
\tikzstyle{cloud} = [draw, ellipse,fill=red!20, node distance=3cm,
    minimum height=2em]
    
\usepackage{enumitem}   
\usepackage{bm}
\usepackage[nomessages]{fp}
\usepackage{afterpage}
\usepackage[ruled,vlined]{algorithm2e}

\SetCommentSty{mycommfont}

\title{Analytical Probability Distributions and EM-Learning for Deep Generative Networks}

\author{%
  Randall Balestriero \\
  ECE Department\\
  Rice University\\
   \And
  Sébastien Paris \\
  Aix-Marseille Univ, Université de Toulon, \\ 
  CNRS, LIS, Toulon, France\\ 
   \And
  Richard~G.~Baraniuk\\
  ECE Department\\
  Rice University \\
}

\begin{document}

\maketitle

\begin{abstract}
Deep Generative Networks (DGNs) with probabilistic modeling of their output
and latent space are currently trained via Variational Autoencoders
(VAEs).
In the absence of a known analytical form for the posterior and
likelihood expectation, VAEs resort to approximations, including
(Amortized) Variational Inference (AVI) and Monte-Carlo (MC) sampling.
We exploit the Continuous Piecewise Affine (CPA) property
of modern DGNs to derive their posterior and marginal
distributions as well as the latter's first moments. These findings enable us 
to derive an analytical Expectation-Maximization (EM) algorithm that enables gradient-free DGN learning.
We demonstrate empirically that EM training of DGNs produces greater
likelihood than VAE training.
Our findings will guide the design of new VAE AVI that better approximate the true posterior and open avenues to appply standard statistical tools for model comparison, anomaly detection, and missing data imputation.
\end{abstract}

\section{Introduction}

Deep Generative Networks (DGNs), which map a low-dimensional latent variable $\bz$ to a higher-dimensional generated sample $\bx$ are the state-of-the-art methods for a range of machine learning applications, including anomaly detection, data generation, likelihood estimation, and exploratory analysis across a wide variety of datasets
\cite{blaauw2016modeling,inoue2018transfer,NIPS2018_8005,lim2018molecular}. 

Training of DGNs roughly falls into two camps: (i) by leveraging an adversarial network as in a Generative Adversarial Network (GAN) \cite{goodfellow2014generative} to turn the method into an adversarial game; and (ii) by modeling the latent variable and observed variables as random variables and performing some flavor of likelihood maximization training. A widely used solution to likelihood based DGN training is via Variational Autoencoders (VAEs) \cite{kingma2013auto}. The popularity of the VAE is due to its intuitive and interpretable loss function, which is obtained from likelihood estimation, and its ability to exploit standard estimation techniques ported from the probabilistic graphical models literature.

Yet, VAEs only offer an {\em approximate} solution for likelihood based training of DGNs. In fact, all current VAEs employ three major approximation steps in the likelihood maximization process. 
First, the true (unknown) posterior is approximated by a variational distribution. This estimate is governed by some free parameters that must be optimized to fit the variational distribution to the true posterior. VAEs estimate such parameters by means of an alternative network, the {\em encoder}, with the datum as input and the predicted optimal parameters as output. This step is referred to as Amortized Variational Inference (AVI), as it removes the explicit, per datum, optimization by a single deep network (DN) pass.
Second, as in any latent variable model, the complete likelihood is estimated by a lower bound (ELBO) obtained from the expectation of the likelihood taken under the posterior or variational distribution. With a DGN, this expectation is unknown and thus VAEs estimate the ELBO by Monte-Carlo (MC) sampling.
Third, the maximization of the MC estimated ELBO, which drives the parameters of the encoder to better model the data distribution and the encoder to produce better variational parameter estimates, is performed by some flavor of gradient descend (GD).

These VAE approximation steps enable rapid training and test-time inference of DGNs.
However, due to the lack of analytical forms for the posterior, ELBO, and explicit (gradient free) parameter updates, it is not possible to measure the above steps' quality or effectively improve them.
Since the true posterior and expectation are unknown, current VAE research roughly fall into three camps: (i) on developing new and more complex output and latent distributions \cite{nalisnick2016stick,li2017collaborative} such as the truncated distribution; (ii) on improving the various estimation steps by introducing complex MC sampling with importance re-weighted sampling \cite{burda2015importance}; on providing different estimates of the posterior with moment matching techniques \cite{dieng2019reweighted,huang2019hierarchical}.

{\bf\em In this paper, we advance both the theory and practice of VAEs by computing the exact analytical posterior and marginal distributions of any DGN employing continuous piecewise affine (CPA) nonlinearities. 
The knowledge of these distributions enables us to perform exact inference without resorting to AVI or MC-sampling and to train the DGN in a gradient-free manner with guaranteed convergence.}

The analytical distributions we obtain provide first-of-their-kind insights into
(i) how DGNs model the data distributions {\` a} la Mixture of Probabilistic Principal Component Analysis (MPPCA), 
(ii) how inference is performed and is akin Generative Latent Optimization models \cite{bojanowski2017optimizing}, 
(iii) the roles of each DGN parameter and how are they updated, and
(iv) the impact of DGN architecture and regularization choice in the form of the DGN distributions and layer weights.
The exact likelihood and marginal computation also enables the use of standard statistical model comparison tools such as the Akkaike Information Criterion (AIC) \cite{akaike1974new} and Bayesian Information Criterion (BIC) \cite{schwarz1978estimating} and inspires new reliable anomaly detection approaches.

Having the exact posterior also enables us to quantify the approximation error of the AVI and MC sampling of VAEs and guide the development of VAEs by leveraging the analytical posterior to design more adapted variational distributions. 
In fact, current VAEs suffer from occasional training instabilities \cite{zhao2019infovae,li-etal-2019-stable}; we validate the empirical observation that VAEs training instabilities emerge from an inadequate variational estimation of the posterior.

We summarize our main contributions as follows:

{\bf [C1]}
We leverage the CPA property of current DGNs to obtain the analytical form of their conditional, marginal, and posterior distributions, which are mixtures of truncated Gaussians and relate DGN density modeling to MPPCA and MFA (Sec.~\ref{sec:posterior}). 
We develop new algorithms and methods to compute the DGN latent space partition, per-region affine mappings, and per-region Gaussian integration (Sec.~\ref{sec:integral}).

{\bf [C2]}
We leverage the analytical form of a DGN's posterior distribution to obtain its first two moments. 
We then leverage these moments to obtain the analytical expectation of the complete likelihood with respect to the DGN posterior (E-step), which enables {\em  encoder-free EM training with guaranteed convergence} (Sec.~\ref{sec:Estep}). 
We also derive the analytical M-step, which enables for the first time {\em guaranteed and rapid gradient-free learning of DGNs} (Sec.~\ref{sec:Mstep}). 
The analytical E-step allows to interpret how the expected latent representation of an input is formed while the M-step demonstrates how information is propagated through layers akin backpropagation encounter in gradient descent.

{\bf [C3]}
We compare our exact E-step to standard VAE training to demonstrate that the VAE inference step is to blame for unstable training.
We also demonstrate how EM-based DGN training provides much faster and stable convergence (Sec.~\ref{sec:experiments}); and provide new directions to leverage the analytical distributions to improve VAE models.
    
Reproducible code for all experiments and figures will be provided on Github at \url{https://github.com/RandallBalestriero/EMDGN.git}.
The proofs of all results are provided in the Supplementary Material. 

\section{Background}
\label{sec:background}

{\bf Max-Affine Spline Deep Generative Networks.}~ 
A deep generative network (DGN) is an operator $g$ that maps a (typically low-dimensional) latent vector $\bz \in \R^S$ to an observation $\bx \in \R^D$ \footnote{Note that we do not require that $S<D$.}
by composing $L$ intermediate {\em layer} mappings $g^{\ell}$, $\ell=1,\dots,L$, that combine affine operators such as the {\em fully connected operator} (simply an affine transformation defined by weight matrix $\bW^{\ell}$ and bias vector $\bv^{\ell}$), {\em convolution operator} (with circulent $\bW^{\ell}$), and, nonlinear operators such as the {\em activation operator} (applying a scalar nonlinearity such as the ubiquitous ReLU), or the {\em (max-)upsampling operator}; definitions of these operators can be found in \cite{goodfellow2016deep}.

In this paper, we focus on DGNs employing arbitrary affine operators and continuous piecewise affine (CPA) nonlinearies, such as the ReLU, leaky-ReLU, and absolute value activations, and spatial/channel max-pooling.
In this case, the entire DGN is the composition of {\em Max-Affine Spline Operators} (MASOs) \cite{balestriero2018spline} and is overall a CPA operator \cite{4560241,arora2016understanding,rister2017piecewise,unser2019representer}. 
As such, DGNs inherit a latent space partition $\Omega$ and a per-region affine mapping 
\begin{align}
g(\bz) = \bA_{\omega} \bz +\bb_{\omega},  \forall \omega \in \Omega,
\label{eq:MASO}
\end{align}
where the per-region slope and bias parameters are a function of the per-layer parameters $\bW_{\ell},\bb_{\ell}$. For various properties of such CPA DGNs, see \cite{balestriero2020max} and for details on the partition, see \cite{balestriero2019geometry}. 
In this paper we will make explicit the per-region affine mappings; to this end,  it is practical to encode the derivatives of the DGN nonlinearities in the matrices $\bQ_{\ell}$. For activation operators, this is a square diagonal matrix with values $\in \{\eta,1\}$ ($\eta >0$ for leaky-ReLU, $\eta = 0$ for ReLU, and $\eta=-1$ for absolute value).
For the max-pooling operator, it is a rectangular matrix filled with $\{0, 1\}$ values based on the pooling $\argmax$. We thus obtain
\begin{align}
    \bA_{\omega} =& \bW^{L}\bQ_{\omega}^{L-1}\bW^{L-1}\dots \bQ_{\omega}^{1}\bW^{1}\;\;\text{ and }\;\;
    \bb_{\omega} = \bv^{L}+\sum_{i=1}^{L-1}\bW^{L}\bQ_{\omega}^{L-1}\bW^{L-1}\dots \bQ_{\omega}^{i}\bv^{i}.
    \label{eq:region_parameters}
\end{align}
Throughout the rest of the paper, the upper index will indicate the layer and not a power.

{\bf Variational Expectation-Maximization.}~ 
A Probabilistic Graphical Model (PGM) combines probability and graph theory into an organized data structure that expresses the relationships between a collection of random variables: the {\em observed} variables collected into $\bx$ and the {\em latent}, or unobserved, variables collected into $\bz$ \cite{jordan2003introduction}.
The parameters $\theta$ that govern the PGM probability distributions are learned from observations $\bx_i \sim \bx, i=1,\dots,N$, requiring estimation of the unobserved $\bz_i,\forall i$. This inference-optimization is commonly done with the Expectation-Maximization (EM) algorithm \cite{dempster1977maximum}.

The EM algorithm consists of (i) estimating each $\bz_i$ from the Expectation of the complete log-density taken with respect to the posterior distribution under the current parameters at time $t$; (ii) Maximizing the estimated complete log-likelihood to produce the updated parameters $\theta_{t+1}$. 
The estimated complete log-likelihood obtained from the E-step is a tight lower bound to the true complete log-likelihood; this lower bound is maximized in the M-step. This process has many attractive theoretical properties, including guaranteed convergence to a local minimum of the likelihood \cite{koller2009probabilistic}.

In the absence of closed form or tractable posterior, an alternative (non-tight) lower bound can be obtained by using a {\em variational distribution} instead. This distribution is governed by parameters $\gamma$ that are optimized to make this distribution as close as possible to the true posterior. This process is results in a {\em variational} E (VE) step \cite{attias2000variational} or {\em variational inference} (VI).
The {\em tightness} of the lower bound is measured by the KL-divergence between the variational and true posterior distributions. Minimization of this divergence cannot be done directly (due to the absence of tractable posterior) but rather indirectly by maximizing the so-called evidence lower bound (ELBO) via
\begin{align}
\textstyle
    \log(p(\bx)) &=  \underbrace{\mathbb{E}_{q(\bz|\gamma)}[\log(p(\bx,\bz|\theta))]+\text{H}(q(\bz|\gamma))}_{\text{ELBO}}+\text{KL}(q(\bz|\gamma) || p(\bz|\bx,\theta)),\label{eq:max_KL}
\end{align}
with $q$ the variational distribution and $H$ the (differential) entropy. Maximization the ELBO with respect to $\gamma$ produces the $\gamma^*$ that adapts $q(\bz|\gamma^*)$ to fit as closely as possible to the true posterior.
Finally, maximizing the ELBO with respect to the PGM parameters $\theta$ provides $\theta_{t+1}$; this can be performed on the entire dataset or on mini-batches \cite{hoffman2013stochastic}.

{\bf Variational AutoEncoders.} A {\em Variational AutoEncoder} (VAE) uses a minimal {\em probabilistic graphical model} (PGM) with just a few nodes but highly nonlinear inter-node relations \cite{lappalainen2000bayesian,valpola2000unsupervised}. The use of DNs to model the nonlinear relations originated in \cite{oh1998learning,ghahramani1999learning,mackay1999density} and has been born again with VAEs \cite{kingma2013auto}. 
Many variants have been developed but the core model consists of modeling the latent distribution over $\bz$ with a Gaussian or uniform distribution and then modeling the data distribution as $\bx = g(\bz)+\bepsilon$ with $\bepsilon$ some noise distribution and $g$ a DGN. Learning the DGN/PGM parameters requires inference of the latent variables $\bz$. This inference is done in VAEs by producing an {\em amortized} VI where a second {\em encoder} DN $f$ produces $\gamma^*_n=f(\bx_n)$ from (\ref{eq:max_KL}). 
Hence, the encoder is fed with an observation $\bx$ and outputs its estimate of the optimal variational parameters that minimizes the KL-divergence between the variational distribution and true posterior. During learning, the encoder adapts to make better estimates $f(\bx_n)$ of the optimum parameters $\gamma_n$. Then, the ELBO is estimated with some flavor of Monte-Carlo (MC) sampling (since its analytical form is not known) and the maximization of the $\theta$ parameters is solved iteratively using some flavor of gradient descent.

\section{Posterior and Marginal Distributions of Deep Generative Networks}
\label{sec:Posterior}

We now derive analytical forms of the key DGN distributions by exploiting the CPA property.
In Sec.~\ref{sec:EM_} we will use this result to derive the EM learning algorithm for DGNs and study the VAE inference approximation versus the analytical one. 

Our key insight is that a CPA DGN consists of an implicit latent space partition and an associated per-region affine mapping (recall (\ref{eq:MASO})). 
In a DGN, propagating a latent datum $\bz$ through the layers progressively builds the $\bA_{\omega},\bb_{\omega}$.
We now demonstrate that turning this region selection process explicit, the analytical DGN marginal and posterior distributions can be obtained.

\subsection{Conditional, Marginal and Posterior Distributions of Deep Generative Networks}
\label{sec:posterior}

Throughout the paper we will consider the commonly employed case of a centered Gaussian latent prior and centered Gaussian noise \cite{zhang2018advances} as
\begin{align}
    p(\bx|\bz) = \phi(\bx;g(\bz),\bSigma_{\bx}),\;p(\bz) = \phi(\bz;0,\bSigma_{\bz}),\label{eq:model}
\end{align}
with $\phi$ the multivariate Gaussian density function with given mean and covariance matrix \cite{degroot2012probability}.
When using CPA DGNs, the generator mapping is continuous and piecewise affine with an underlying latent space partition, and per-region mapping as given by (\ref{eq:MASO}). We can thus obtain the analytical form of the conditional distribution of $\bx$ given the latent vector $\bz$ as follows.

\begin{lemma}
\label{lemma:conditional}
The DGN conditional distribution is given by 
$
p(\bx|\bz) = \sum_{\omega \in \Omega}\Indic_{\bz \in \omega} \phi\left(\bx;\bA_{\omega}\bz + \bb_{\omega},\bSigma_{\bx} \right)
$
with per-region parameters from (\ref{eq:region_parameters}).
\end{lemma}

This type of data modeling is closely related to
MPPCA \cite{tipping1999mixtures} that combines multiple PPCAs \cite{tipping1999probabilistic} and MFA \cite{ghahramani1996algorithm,hinton1997modeling} that combines multiple factor analyzers \cite{harman1976modern}. 
The associated PGMs represent the data distribution with $R$ components and leverage an explicit categorical distribution $\bt\sim Cat(\pi)$, leading to the conditional input distributions
$
    \bx|(\bz,\bt) = \sum_{r=1}^R\Indic_{r=\bt}\left(\bW_r \bz + \bv_r \right)+\bepsilon,
$ 
with $\bW_r,\bv_r$ denoting the per-component affine parameters and with $\bSigma_{\bx}$ diagonal (MPPCA) or fully occupied (MFA) and $\bz \sim \mathcal{N}(\bmu_{\bz},\bSigma_{\bz})$. Note, however, that neither MPPCA nor MFA impose continuity in the $(\bt,\bz)\mapsto \bx$ mapping as opposed to a DGN. To formalize this, consider an (arbitrary) ordering of the DGN latent space regions as $\omega_1,\dots, \omega_R$ with $R=\card(\Omega)$; we also denote by $\Phi_{\omega}$ the cumulative density function on $\omega_r$ (integral of the density function on $\omega_r$).

\begin{prop}
\label{prop:mppca}
A DGN with distributions given by (\ref{eq:model}) corresponds to a continuous  MPPCA (or MFA)
model with implicit categorical variable given by $p(\bt=r)= \Phi_{\omega_r}(\mathbf{0},\bSigma_{\bz}), \bW_r=\bA_{\omega_r}, \bv_r=\bb_{\omega_r}, R=\card(\Omega)$ and $\bSigma_{\bx}=\sigma I$ (or full $\bSigma_{\bx}$).
\end{prop}

Note that this result generalizes the result of \cite{lucas2019understanding} which related  linear and shallow DGNs to PPCA, as in the linear regime one has $g(\bz)=\bW\bz+\bb+\bepsilon$.
We now consider the marginal $p(\bx)$ and  posterior $p(\bz|\bx)$ distributions. The former will be of use to compute the likelihood, while the latter will enable us to derive the analytical E-step in the next section.

\begin{thm}
\label{thm:px}
The marginal and  posterior distributions of a CPA DGN are given by
\begin{align}
p(\bx)=&\sum_{\omega\in\Omega}\phi(\bx; \bb_{\omega},\bSigma_{\bx}+\bA_{\omega}\bSigma_{\bz}\bA_{\omega}^T)\Phi_{\omega}(\bmu_{\omega}(\bx),\bSigma_{\omega}),\label{eq:marginal}\\
p(\bz|\bx)=&p(\bx)^{-1}\sum_{\omega\in\Omega}\Indic_{\bz \in \omega} 
\phi(\bx;\bb_{\omega},\bSigma_{\bx}+\bA_{\omega}\bSigma_{\bz}\bA_{\omega}^T)\phi(\bz;\bmu_{\omega}(\bx),\bSigma_{\omega}),\label{eq:posterior}
\end{align}
with
\begin{gather}
    \bmu_{\omega}(\bx)=\bSigma_{\omega}\left(\bA_{\omega}^T\bSigma^{-1}_{\bx}(\bx-\bb_{\omega})\right),\;\;
\bSigma_{\omega}=\left(\bSigma^{-1}_{\bz}+\bA^T_{\omega}\bSigma^{-1}_{\bx}\bA_{\omega}\right)^{-1}\label{eq:mu_sigma}.
\end{gather}
\end{thm}

The distribution $\Phi_{\omega}(\bmu_{\omega}(\bx),\bSigma_{\omega})$ is derived in the next section.
For both the marginal and the posterior distribution, on each partition region, there exists a mean $\bmu_{\omega}(\bx)$ and covariance $\bSigma_{\omega}$ that we can interpret.
For that purpose, consider $\bSigma_{\bx}=I, \bSigma_{\bz}=I$ to 
obtain $\bmu_{\omega}(\bx)=(I+\bA^T_{\omega}\bA_{\omega})^{-1}\bA_{\omega}^T(\bx-\bb_{\omega})$. That is, the bias of the per-region affine mapping is removed from the input which is then mapped back to the latent space via $\bA_{\omega}^T$ and whitened by the ``regularized'' inverse of the correlation matrix of $\bA_{\omega}$. Note that $\bA_{\omega}^T$  backpropagates the signal from the output to the latent space in the same way that gradients are backpropagated during gradient learning of a DN. We further highlight the specific form of the posterior as being a mixture model of truncated Gaussians \cite{horrace2005some}, a truncated Gaussian being a Gaussian distribution for which the domain $\mathbb{R}^S$ has been constrained to a (convex) sub-domain, $\omega$ in our case.

\begin{prop}
\label{prop:mppca_posterior}
The DGN posterior distribution is a mixture of $\card(\Omega)$ truncated Gaussians, each truncated on a different polytope $\omega \in \Omega$ with mean $\bmu_{\omega}(\bx)$ and covariance $\bSigma_{\omega}$ from (\ref{eq:mu_sigma}).
\end{prop}

{\bf Zero-Noise Limit and Generative Latent Optimization (GLO) Models.}~ 
In the zero-noise limit ($\bSigma_{\bx}=\sigma I$ and $\sigma \rightarrow 0$) 
the posterior
takes a very special form. Denote by $\bz^*(\bx)\triangleq \argmin_{\bz} \|\bx-g(\bz) \|_2^2+\bz^T\bSigma_{\bz}\bz$, the (regularized) latent vector that produces the closest output from an observation $\bx$.

\begin{lemma}
\label{lemma:zeronoise}
In the zero-noise limit, the DGN posterior distribution converges to a Dirac positioned in the $\bz$-space at $\bz^*(\bx)$ as $
    \lim_{\sigma \rightarrow 0}p(\bz|\bx) = \delta(\bz-\bz^*(\bx))$. 
\end{lemma}

Interestingly, GLO \cite{bojanowski2017optimizing} performs DGN training by first inferring a latent vector akin to $\bz^*(\bx)$ but without the $\ell_2$ regularization $\bz^T\bSigma_{\bz}\bz$, which is often replaced by a $\bz$ truncation/clipping. 

\begin{prop}
\label{prop:GLO}
The GLO-inferred DGN latent variable associated to an observation $\bx$ corresponds to the maximum a posteriori estimate of the zero-noise limit posterior distribution and with uninformative prior (large $\bSigma_{\bz}$) or with uniform prior $\bz \sim \mathcal{U}([a,b])$ when using $[a,b]$ clipping).
\end{prop}

\subsection{Gaussian Integration on the Deep Generative Network Latent Partition}
\label{sec:integral}

We now turn to the computation of the DGN marginal (\ref{eq:marginal}) and posterior (\ref{eq:posterior}) distributions for which we need to integrate over all of hte the regions $\omega \in \Omega$ in the latent space partition.


{\bf Obtaining the DGN Partition.}~ 
Each region $\omega \in \Omega$ is a polytope that can be explicitly described via a system of inequalities involving the up-to-layer $\ell$ mappings
\begin{align}
    \bA^{1\rightarrow \ell}_{\omega}&\triangleq \bW^{\ell}\bQ^{\ell-1}_{\omega}\bW^{\ell-1}\dots \bQ^{1}_{\omega}\bW^{1}\;\text{ and }\;
    \bb^{1\rightarrow \ell}_{\omega}\triangleq \bv^{\ell}+\sum_{i=1}^{\ell-1}\bW^{\ell}\bQ^{\ell-1}_{\omega}\bW^{\ell-1}\dots \bQ^{i}_{\omega}\bv^{i} ,\label{eq:Abell}
\end{align}
producing the pre-activation feature maps  $\bh^{\ell}(\bz)\in \mathbb{R}^{D^{\ell}}$ by 
$
    \bh^{\ell}(\bz)= \bA^{1\rightarrow \ell}_{\omega}\bz+\bb^{1\rightarrow \ell}_{\omega}
$ and with $\bA^{1\rightarrow \ell}_{\omega}\in \mathbb{R}^{D^{\ell}\times S}$ and $\bb^{1\rightarrow \ell}_{\omega} \in \mathbb{R}^{D^\ell}$.
Note that we have, in particular, that $\bA^{L}_{\omega}=\bA_{\omega}$ and $\bb^{L}_{\omega}=\bb_{\omega}$ from (\ref{eq:region_parameters}). 
When using standard activation functions such as (leaky-)ReLU or absolute value, the sign of the pre-activation defines the activation state; denote this by $\bq^{\ell}=\sign(\bh^{\ell}(\bz))$ and collect all of the per-layer signs into $\bq$. Without degenerate weights, the sign patterns produced by $\bq(\bz),\forall \bz$ are tied to the regions $\omega \in \Omega$; we will thus use interchangeably $\bq(\bz)$ with $\bz \in \omega$ and $\bq(\omega)$.

\begin{lemma}
\label{lemma:bijection}
The operator $\bz \mapsto [\bq^1(\bz),\dots,\bq^{L-1}(\bz)]$ is piecewise constant with a bijection between its image and $\Omega$.
\end{lemma}

\begin{cor}
\label{cor:H_rep}
The polyhedral region $\omega$ is given by
\begin{align*}
    \omega =\bigcap_{\ell=1}^{L-1}\left\{\bz \in \mathbb{R}^S:\bA^{1\rightarrow \ell}_{\omega}\bz <- \bq^{\ell}(\omega) \odot \bb^{1\rightarrow \ell}_{\omega}\right\},
\end{align*}
with $\odot$ the Hadamard product.
\end{cor}

\begin{figure}[t]
    \centering
    \includegraphics[width=1\linewidth]{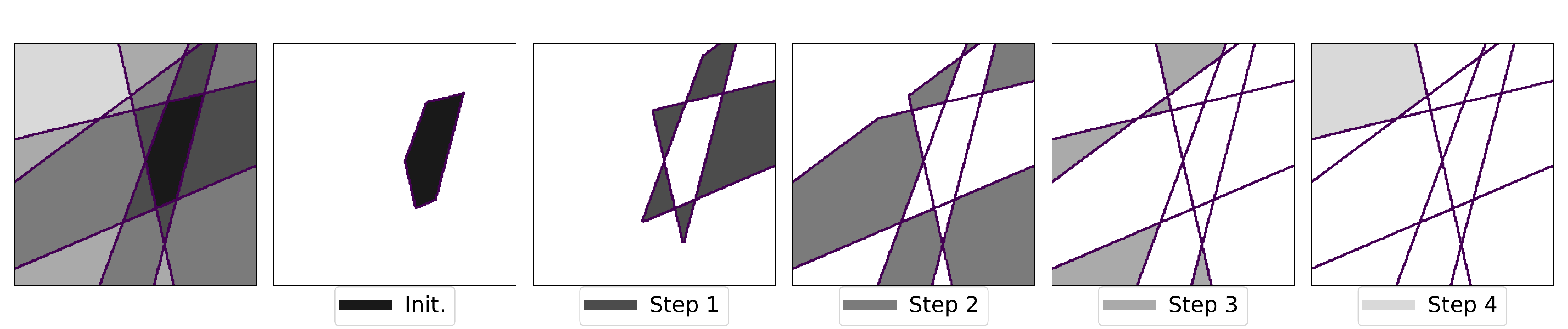}
    \caption{\small 
    Recursive partition discovery for a DGN with $S=2$ and $L=2$, starting with an initial region obtained from a sampled latent vector $\bz$ (init). By walking on the faces of this region, neighboring regions sharing a common face are discovered (Step 1). Recursively repeating this process until no new region is discovered (Steps 2--4) provides the DGN latent space partition at left .\normalsize }
    \vspace*{-3mm}
    \label{fig:partition}
\end{figure}

The above result tells us that the pre-activation signs locate on which side of each hyperplane the region $\omega$ is located, which provides a direct way to compute the $\mathcal{H}$-representation of $\omega$ from $\bq(\bz)$ with $\bz \in \omega$. 
To obtain the entire partition $\Omega$, we propose a recursive scheme that starts from an initial region (or sample $\bz$) and walks on its faces to discover the neighboring regions. This process is repeated on the newly discovered regions until no new region is discovered.
We detail this exploration procedure in Appendix~\ref{appendix:latent_partition} and illustrate it in Fig.~\ref{fig:partition}.

{\bf Gaussian Integration on $\omega$.}~ 
The Gaussian integral on a region $\omega$ (and its moments) cannot in general be obtained by direct integration unless $\omega$ is a rectangular region \cite{tallis1961moment,bg2009moments} or is polytopal with at most $S$ faces \cite{tallis1965plane}. 
In general, the DGN regions $\omega \in \Omega$ will have at least $S+1$ faces, as they are closed polytopes in $\mathbb{R}^S$. To leverage the known integral forms, we propose to first decompose a DGN region $\omega$ into simplices ($S+1$-face polytopes in our case \cite{munkres2018elements}) and then further decompose each simplex into open polytopes with at most $S$ faces, which enables the use of \cite{tallis1965plane}. In our case, we perform the simplex decomposition with the Delaunay  triangulation \cite{delaunay1934sphere} denoted as $T(\omega)$ with
\begin{align}
    T(\omega) \triangleq \{ \Delta_1,\dots, \Delta_{\card(T(\omega))} \},\text{ with } \cup_{i=1}^{\card(T(\omega))}\Delta_i = \omega \text{ and }\Delta_i \cap \Delta_j = \emptyset, \forall i  \not = j,
    \label{eq:S}
\end{align}
where each $\Delta_i$ is a simplex defined by the half-spaces $\Delta_i=\cap_{s=1}^{S+1}H_{i,j}$. This process is illustrated in Fig.~\ref{fig:region_to_simplex}. The decomposition of each simplex into open polytopes with less than $S+1$ faces is performed by employing the standard inclusion-exclusion principle \cite{bjorklund2009set} leading to the following result.

\begin{figure}[t]
    \centering
    \begin{minipage}{0.6\linewidth}
    \centering
    \includegraphics[width=1\linewidth]{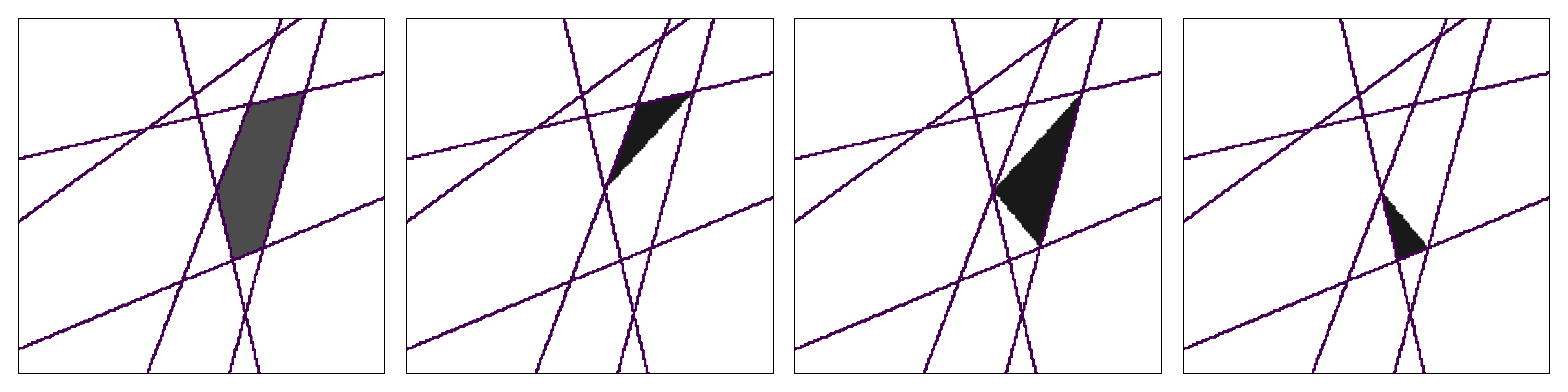}
    \end{minipage}
    \begin{minipage}{0.38\linewidth}
    \caption{\small Triangulation $T(\omega)$ as per (\ref{eq:S}) of a polytopal region $\omega$ (left plot) obtained from the Delaunay Triangulation of the region vertices leading to $3$ simplices (three right plots).\normalsize}
    \label{fig:region_to_simplex}
    \end{minipage}
    \vspace*{-5mm}
\end{figure}

\begin{lemma}
\label{lemma:integral_decomposition}
The integral of any integrable function $g$ on a polytopal region $\omega \in \Omega$ can be decomposed into integration over open polytopes of at most $S$ faces via
\begin{align*}
    \int_{\omega}  g(\bz) d\bz =  \sum_{\Delta \in T(\omega)}\sum_{(s,V) \in H(\Delta)} s\int_{V}g(\bz)d\bz.
\end{align*}
with $
    H(\Delta_i)\triangleq \left\{\left((-1)^{|J|+S}, \cap_{j \in J}H_{i,j}\right),J \subseteq \{1,\dots,S+1\}, |J|\leq S\right\}$.
\end{lemma}

From the above result, we can apply the known form of the Gaussian integral on a polytopal region with fewer than $S$ faces and obtain the form of the integral and moments as provided in Appendix~\ref{appendix:analytical_form}, where detailed pseudo code is provided.

{\bf Remark:} The integral performed as per Lemma~\ref{lemma:integral_decomposition} is computationally expensive, particularly with respect to the latent space dimension $S$. This is the current main practical limitation of performing the analytical computation of the DGN posterior (and thus the E-step). 
A more elaborated discussion plus several solutions are provided in the next section; see also Appendix~\ref{appendix:complexity} for the asymptotic computational complexity details.

{\bf Visualization of the Marginal and Posterior Distributions.}~ 
To illustrate our theoretical development so far, we now visualize the posterior and marginal distributions of a randomly initialized DGN in a low-dimensional space $D=2$ and with latent dimension $S=1$. 
(See Appendix~\ref{appendix:more_XP} for the architectural details of the DGN.)
We depict the obtained distributions as well as the generated samples in Fig.~\ref{fig:toy_viz}. We also plot the posterior distribution b
ased on one observation obtained via $g(\bz_0)$ given a sampled $\bz_0$ from the $\bz$ distribution and one noisy observation $g(\bz_0)+\bepsilon_0$ given a noise realization $\bepsilon_0$.

\begin{figure}[t]
    \centering
    \begin{minipage}{0.55\linewidth}
    \includegraphics[width=0.32\linewidth]{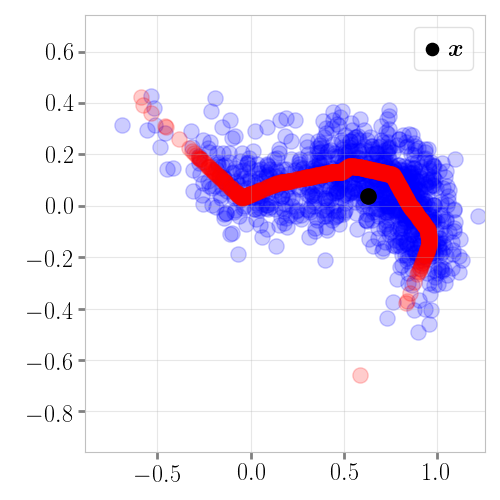}
    \includegraphics[width=0.32\linewidth]{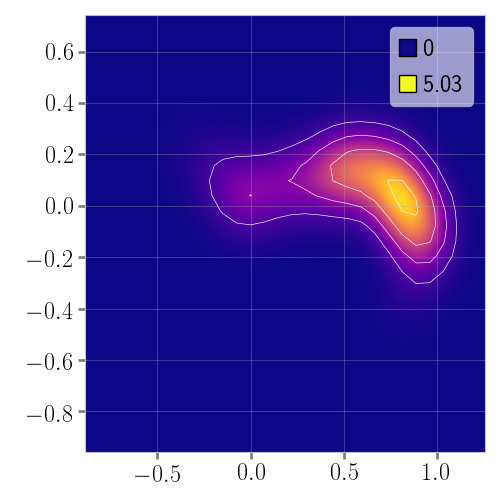}
    \includegraphics[width=0.32\linewidth]{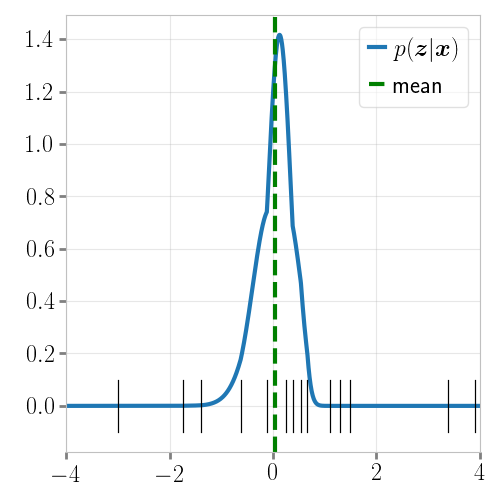}
    \end{minipage}
    \begin{minipage}{0.44\linewidth}
    \caption{{\small {\bf Left:}~Noiseless generated samples $g(\bz)$ in red and noisy samples $g(\bz)+\bepsilon$ in blue, with $\bSigma_{\bx}=0.1I$, $\bSigma_{\bz}=I$. {\bf Middle:}~  marginal distribution $p(\bx)$ from (\ref{eq:marginal}). {\bf Right:}~ the posterior distribution $p(\bz|\bx)$ from (\ref{eq:posterior}) (blue), its expectation (green) and the position of the region limits (black), with sample point $\bx$ depicted in black in the left figure.}}
    \label{fig:toy_viz}
    \end{minipage}
    \vspace*{-5mm}
\end{figure}

\section{Expectation-Maximization Learning of Deep Generative Networks}
\label{sec:EM_}

We now derive an analytical Expectation-Maximization training algorithm for CPA DGNs based on the results of the previous sections. 
We then compare DGN training via EM and AVI and leverage the exact complete likelihood to perform model selection and study the VAE approximation error.

\subsection{Expectation Step}
\label{sec:Estep}

The E-step infers the latent (unobserved) variables associated to the generation of each observation $\bx$ by taking the expectation of the log of the complete likelihood with respect to the posterior distribution (\ref{eq:posterior}). 
We denote the per-region moments of the DGN posterior (from Appendix~\ref{appendix:analytical_form}) by $\mathbb{E}_{\bz|\bx}[\Indic_{\bz\in\omega}]\triangleq e^0_{\omega}(\bx)$, $\mathbb{E}_{\bz|\bx}[\bz\Indic_{\bz\in\omega}]\triangleq \be^1_{\omega}(\bx)$ and $\mathbb{E}_{\bz|\bx}[\bz\bz^T\Indic_{\bz\in\omega}] \triangleq \bE^2_{\omega}(\bx)$; we also have $\be^1(\bx)\triangleq \mathbb{E}_{\bz|\bx}[\bz] = \sum_{\omega}\be^1_{\omega}(\bx)$ and likewise for the second moment.
We obtain the following E-step (the detailed derivations are in Appendix \ref{proof:Estep}) 
\begin{align*}
    E_{\bz|\bx}\left[\log \left( p(\bx|\bz)p(\bz) \right)\right]=&-\frac{1}{2}\log\Big( (2\pi)^{S+D}|\det(\bSigma_{\bx})||\det(\bSigma_{\bz})| \Big)-\frac{1}{2}\trace(\bSigma_{\bz}^{-1}\bE^2(\bx))\\
    &-\frac{1}{2}\Bigg(\bx^T\bSigma_{\bx}^{-1}\bx-2 \bx^T\bSigma_{\bx}^{-1}\bigg(\sum_{\omega}\bA_{\omega}\be^1_{\omega}(\bx)+\bb_{\omega}e^0_{\omega}(\bx)\bigg)\\
    &+\sum_{\omega}\bigg[\trace(\bA_{\omega}^T\bSigma_{\bx}^{-1}\bA_{\omega}\bE^2_{\omega}(\bx))+ (e^0_{\omega}\bb_{\omega}+2\bA_{\omega}\be_{\omega}^1(\bx))^T\bSigma_{\bx}^{-1} \bb_{\omega}\bigg]\Bigg).
\end{align*}
Note that the (per-region) moments involved in the E-step, such as $\be^1_{\omega}(\bx)$, are taken with respect to the current parameters ($\theta=\{\bSigma_{\bx},\bSigma_{\bz},(\bW^{\ell},\bv^{\ell})_{\ell=1}^L\})$.
That is, if gradient based optimization is leveraged to maximize the ELBO, then no gradient should be propagated through them. 
We can see from the above formula that the contributions of each region's affine parameters are weighted based on the posterior for each datum $\bx$. That is, for each input $\bx$, the posterior combines all of the per-region affine parameters as opposed to current forms of learning that only leverage the parameters involved on the specific region activated by the DGN input $\bz$.

\subsection{Maximization Step}
\label{sec:Mstep}

Given the E-step, maximizing the ELBO can be done via some flavor of gradient based optimization. However, thanks to the analytical E-step and the Gaussian form of the involve distributions, there exists analytical form of this maximisation process (M-step) leading to the analytical M-step for DGNs. The formulas for all of the DGN parameters are provided in Appendix \ref{proof:EM}. We provide here the analytical form for the bias ${\bv^{\ell}}^*$, for which we introduce $\br_{\omega}^{\ell}(\bx)$ as the expected reconstruction error of the DGN as
\begin{gather*}
    \br_{\omega}^{\ell}(\bx) \triangleq \left(\bx-\sum_{i\not = \ell}\bA^{i + 1\rightarrow L}_{\omega}\bQ^{i}_{\omega}\bv^{i}\right) e^0_{\omega}(\bx)-\bA_{\omega}\be^1_{\omega}(\bx)\;\;(\text{expected residual without  $\bv^\ell$}),\\
    {\bv^{\ell}}^* =\left(\sum_{\bx}\sum_{\omega}\bQ^{\ell}_{\omega}\bA^{L\rightarrow \ell+1}_{\omega}\bSigma_{\bx}^{-1}\bA^{\ell+1\rightarrow L}_{\omega}\bQ^{\ell}_{\omega}\right)^{-1}\hspace{-0.13cm}\left(\sum_{\bx}\sum_{\omega \in \Omega} \underbrace{\bQ^{\ell}_{\omega}\bA^{L\rightarrow \ell+1}_{\omega}\bSigma_{\bx}^{-1}  \br^{\ell}_{\omega}(\bx)}_{\text{residual back-propagated to layer $\ell$}}\right).
\end{gather*}
Some interesting observations can be made based on the analytical form of these updates. First, the bias update is based on the residual of the reconstruction error with a DGN whose bias has been removed; this residual is then backpropagated to the $\ell^{\rm th}$ layer. 
The backpropagation is performed via the (transposed) backpropagation matrix as when performing gradient-based learning. 
Second, the updates of any parameter depend on each region parameter's contribution based on the posterior moments and integrals, similarly to any mixture model. 
Third, all of the updates are whitened based on the backpropagation (or forward propagation) correlation matrix $\bA_{\omega}^{\ell\rightarrow L},\forall \omega, \forall \ell$.
We study the impact of using a probabilistic priors on the layer weights
such as Gaussian, Laplacian, or uniform which are related to the $\ell_2$, $\ell_1$ regularization, and weight clipping techniques in Appendix~\ref{appendix:regularization}.

\begin{figure}[t]
    \centering
    \begin{minipage}{0.5\linewidth}
    \includegraphics[width=\linewidth]{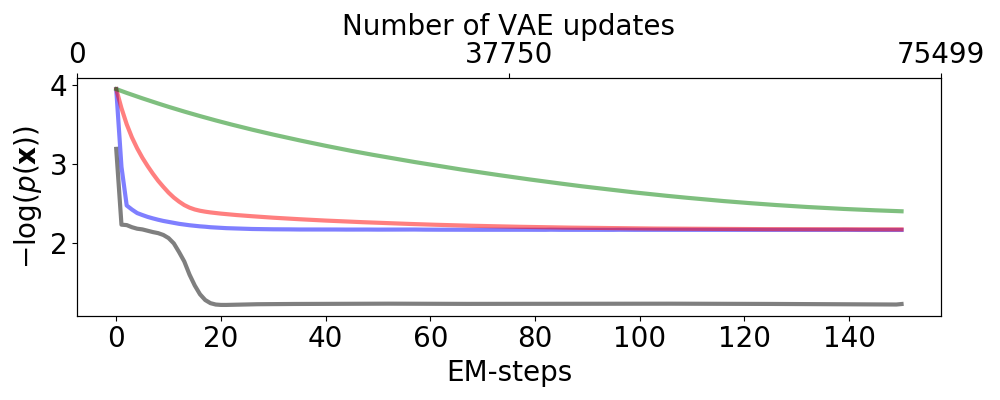}
    \end{minipage}
    \begin{minipage}{0.49\linewidth}
    \caption{{\small DGN training under EM (black) and VAE training with various learning rates for VAE (blue: 0.005, red: 0.001, green: 0.0001). In all cases, VAE converges to the maximum of its ELBO. The gap between the VAE and EM curves is due to the inability of the VAE's AVI to correctly estimate the true posterior, pushing the VAE's ELBO far from the true log-likelihood (recall (\ref{eq:max_KL})) and thus preventing it from precisely approximating the true data distribution.}}
    \label{fig:circle}
    \end{minipage}
\end{figure}

\begin{figure}[t]
    \centering
    \begin{minipage}{0.45\linewidth}
    \includegraphics[width=\linewidth]{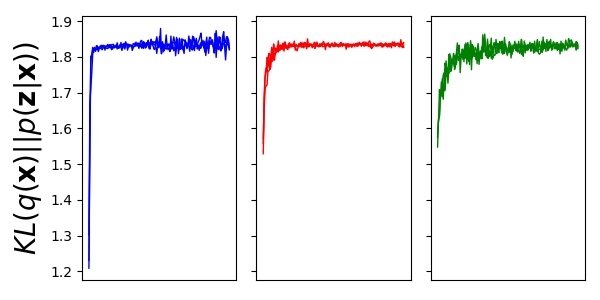}
    \end{minipage}
    \begin{minipage}{0.54\linewidth}
    \caption{{\small KL-divergence between a VAE variational distribution and the true DGN posterior when trained on a noisy circle dataset in $2D$ for 3 different learning rates. During learning, the DGN adapts such that $g(\bz)+\bepsilon$ models the data distribution based on the VAE's estimated ELBO. As learning progresses, the true DGN posterior becomes harder to approximate by the VAE's variational distribution in the AVI process. As such, even in this toy dataset the commonly employed Gaussian variational distribution is not rich enough to capture the multimodality of $p(\bz|\bx)$ from (\ref{eq:posterior}).}}
    \label{fig:KL}
    \vspace*{-3mm}
    \end{minipage}
\end{figure}

\subsection{Empirical Validation and VAE Comparison}
\label{sec:experiments}

We now numerically validate the above EM-steps on a simple problem involving data points on a radius $1$ circle in 2D augmented with a Gaussian noise of standard deviation $0.05$. 
We depict the EM-training of a 2-layer DGN with width of $8$ with the VAE training. In all cases the DGNs are the same architecture, with same weight initialization, and the dataset is also identical between models with the same noise realizations. Thanks to the analytical form of the marginals, we can compute the true ELBO (without variational estimation of the true posterior) for the VAE during its training to monitor its ability to fit the data distribution. We depict this in Fig.~\ref{fig:circle}. 
We observe how EM training converges faster to a lower negative log-likelihood. 
In addition, we see how all of the trained VAEs seem to converge to the same bound, which corresponds to the maximum of its ELBO, where the gap is induced by the use of a variational approximation of the true posterior. 
We further investigate this phenomenon in Fig.~\ref{fig:KL}. 
In fact, we obtained from (\ref{eq:posterior}) and Prop.~\ref{prop:mppca_posterior} that the posterior is a mixture of truncated Gaussian and the covariances are based on $\bA_{\omega}^T\bA_{\omega}$. Hence VAEs employing a single Gaussian as a variational distribution $q(\bx)$ will fail to correctly approximate the true posterior $p(\bz|\bx)$, especially as training proceeds making the posterior covariance less diagonal. The absence of precise posterior estimation directly translated into VAE optimizing an ELBO objective being far from the tight log-likelihood bound, in turn hurting VAE ability to correctly approximate the data distribution as opposed to EM learning producing much smaller negative log-likelihood. Providing better posterior estimates is thus key to improve VAE performances.
From this we conclude that multimodal variational distributions should be considered for VAEs regardless of the data at hand.
We provide additional figures and experiments with various architectures in Appendix~\ref{appendix:more_XP}.

\section{Conclusions}

We have derived the analytical form of the posterior, marginal, and conditional distributions of for CPA DGNs. This has enabled us to derive the EM-learning algorithm for DGNs that not only converges faster than state-of-the-art VAE training but also to a higher likelihood.
Our analytical forms can be leveraged to improve the variational distribution of VAEs, understand the form of analytical weight updates, study how a DGN infers the latent variable $\bz$ from $\bx$, and leverage standard statistical tools to perform model selection, anomaly detection and beyond. 

In addition of these insights, our new EM learning algorithm can be used to train DGNs. The main computational bottlenecks are the width/depth of the DGN and the dimension of the latent space. 
One interesting research direction is developing approximations to our various analytical forms that enable fast inference/learning {\`a} la VAE but with better approximation and quantified approximation error.
One promising direction would involve only computing and integrating the distributions on the regions $\omega \in \Omega$ where we know a priori that the posterior is nonzero.

\newpage
\section*{Broader Impacts}

We have derived the analytical form of the posterior, marginal and conditional distributions of Deep Generative Networks (DGNs) based on continuous piecewise affine architectures. 
Our approach provides an approximation-free alternative to VAEs to train DGNs.
In addition to improving DGN algorithms, our analytical forms will enable researchers to probe more deeply into the inner workings of DGNs and VAE, making them more interpretable and thus trustworthy.
Our calculations will also enable accurate anomaly detection and model selection, which should find wide application in  sensitive applications where accurately computing the probability of a data point is crucial.

\section*{Acknowledgments}
RB and RB were supported by 
NSF grants CCF-1911094, IIS-1838177, and IIS-1730574; 
ONR grants N00014-18-12571 and N00014-17-1-2551;
AFOSR grant FA9550-18-1-0478; 
DARPA grant G001534-7500; and a 
Vannevar Bush Faculty Fellowship, ONR grant N00014-18-1-2047.

\small
\bibliography{references}
\bibliographystyle{plain}

\normalsize

\clearpage
\appendix
\onecolumn

\newgeometry{left=1.5cm,right=1.5cm}

\Huge
\begin{center}
    Supplementary Material
\end{center}
\normalsize
\vspace{1cm}

\section{Computing the Latent Space Partition}
\label{appendix:latent_partition}

In this section we first introduce notations and demonstrate how to express a region $\omega$ of the partition $\Omega$ as a polytope defined by a system of inequalities, and then leverage this formulation to demonstrate how to obtain $\Omega$ by recursively exploring neighboring regions starting from a random point/region.

{\bf Regions as Polytopes}
To represent the regions $\omega \in \Omega$ as a polytope via a system of inequalities we need to recall from (\ref{eq:MASO}) that the input-output mapping is defined on each region by the affine parameters $A_{\omega},B_{\omega}$ themselves obtained by composition of MASOs. Each layer pre-activation (feature map prior application of the nonlinearity) is denoted by $\bh^{\ell}(\bz)\in \mathbb{R}^{D^{\ell}}, \ell=1,\dots,L-1$ and given by 
$
    \bh^{\ell}(\bx)= \bA^{1\rightarrow \ell}_{\omega}\bz+\bb^{1\rightarrow \ell}_{\omega},
$
with  up-to-layer $\ell$ affine parameters
\begin{align}
    \bA^{1\rightarrow \ell}_{\omega}&\triangleq \bW^{\ell}\bQ^{\ell-1}_{\omega}\bW^{\ell-1}\dots \bQ^{1}_{\omega}\bW^{1},&&\bA^{1\rightarrow \ell}_{\omega}\in \mathbb{R}^{D^{\ell}\times S},\label{eq:Aell}\\
    \bb^{1\rightarrow \ell-1}_{\omega}&\triangleq \bv^{\ell}+\sum_{i=1}^{\ell}\bW^{\ell}\bQ^{\ell-1}_{\omega}\bW^{\ell-1}\dots \bQ^{i}_{\omega}\bv^{i}, &&\bb^{1\rightarrow \ell}_{\omega} \in \mathbb{R}^{D^\ell},\label{eq:bell}
\end{align}
which depend on the region $\omega$ in the latent space \footnote{looser condition can be put as the up-to-layer $\ell$ mapping is a CPA on a coarser partition than $\Omega$ but this is sufficient for our goal.}.
Notice that we have in particular $\bA^{L}_{\omega}=\bA_{\omega}$ and $\bb^{L}_{\omega}=\bb_{\omega}$, the entire DGN affine parameters from (\ref{eq:region_parameters}) on region $\omega$. The regions depend on the signs of the pre-activations defined as  $\bq^{\ell}(\bz)=\sign(\bh^{\ell}(\bz))$ due to the used activation function behaving linearly as long as the feature maps preserve the same sign. This holds for (leaky-)ReLU or absolute value, for max-pooling we would need to look at the argmax position of each pooling window, as pooling is rare in DGN we focus here on DN without max-pooling; let $\bq^{\rm all}(\bz)\triangleq [(\bq^{L-1}(\bz))^T,\dots,(\bq^{1}(\bz))^T]^T$ collect all the per layer sign operators without the last layer as it does not apply any activation.

\qq
\begin{lemma}
The $\bq^{\rm all}$ operator is piecewise constant and there is a bijection between $\Omega$ and $\im(\bq)$. 
\end{lemma}
\qq

The above demonstrates the equivalence of knowing $\omega$ in which an input $\bz$ belongs to and knowing the sign pattern of the feature maps associated to $\bz$; we will thus use interchangeably $\bq^{\rm all}(\bz), \bz \in \omega$ and $\bq^{\rm all}(\omega)$.
From this, we see that the pre-activation signs and the regions are tied together. We can now leverage this result and provide the explicit region $\omega$ as a polytope via its sytem of inequality, to do so we need to collect the per-layer slopes and biases into
\begin{align}
\bA^{\rm all}_{\omega}=\begin{bmatrix}\bA^{1\rightarrow L-1}_{\omega}\\
\dots\\
\bA^{1\rightarrow 1}_{\omega}\end{bmatrix},\;
\bb^{\rm all}_{\omega}=\begin{bmatrix}\bb^{1\rightarrow L-1}_{\omega}\\
\dots\\
\bb^{1\rightarrow 1}_{\omega}\end{bmatrix}, \text{ $\bA^{\rm all}_{\omega} \in \mathbb{R}^{(\prod_{\ell=1}^{L-1}D^{\ell})\times S}$, $\bb^{\rm all}_{\omega}\in\mathbb{R}^{\prod_{\ell=1}^{L-1}D^{\ell}}$}.\label{eq:Aall}
\end{align}

\qq
\begin{cor}
\label{cor:H_rep}
The $\mathcal{H}$-representation of the polyhedral region $\omega$ is given by
\begin{align}
    \omega = \{\bz \in \mathbb{R}^S:\bA^{\rm all}_{\omega}\bz < -\bq^{\rm all}(\omega) \odot \bb^{\rm all}_{\omega}\}=\bigcap_{\ell=1}^{L-1}\{\bz \in \mathbb{R}^S:\bA^{1\rightarrow \ell}_{\omega}\bz <- \bq^{\ell}(\omega) \odot \bb^{1\rightarrow \ell}_{\omega}\},\label{eq:ineq}
\end{align}
with $\odot$ the Hadamard product.
\end{cor}
\qq

From the above, it is clear that the sign locates in which side of each hyperplane the region is located. We now have a direct way to obtain the polytope $\omega$ from its sign pattern $\bq^{\rm all}(\omega)$ or equivalently from an input $\bz \in \omega$; the only task left is to obtain the entire partition $\Omega$ collecting all the DN regions, which we now propose to do via a simple scheme.

{\bf Partition Cells Enumeration.}
The search for all cells in a partition is known as the cell enumeration problem and has been extensively studied in the context of speicific partitions such as hypreplane arrangements \cite{avis1996reverse,sleumer1999output,gerstner2006algorithms}. In our case however, the set of inequalitites of different regions changes. In fact, for any neighbour region, not only the sign pattern $\bq^{\rm all}$ will change but also $\bA^{\rm all}_{\omega}$ and $\bb^{\rm all}_{\omega}$ due to the composition of layers. In fact, changing one activation state say $-1$ to $1$ for a specific unit at layer $\ell$ will alter the affine parameters from (\ref{eq:Aell}) and (\ref{eq:bell}) due to the layer composition.
As such, we propose to enumerate all the cells $\omega \in \Omega$ with a deterministic algorithm that starts from an intial region and recursively explores its neighbouring cells untill all have been visited while recomputing the inequality system at each step. 
To do so, consider the initial region $\omega_{0}$. First, one finds all the non-redundant inequalities of the inequality system (\ref{eq:ineq}), the remaining inequalities define the faces of the polytope $\omega$. Second, one obtains any of the neighbouring regions sharing a face with $\omega_0$ by switching the sign in the entry of $\bq(\omega_0)$ corresponding to the considered face. Repeat this for all non-redundant inequalities to obtain all the adjacent regions to $\omega_0$ sharing a face with it. Each altered code defines an adjacent region and its sytem of inequality can be obtain as per Lemma~\ref{cor:H_rep}. Doing so for all the faces of the initial region and then iterating this process on all the newly discovered regions will enumerate the entire partition $\Omega$. We summarize this in Algo~\ref{algo:region} in the appendix and illustrate this recursive procedure in Fig.~\ref{fig:partition}.

We now have each cell as a polytope and enumerated the partition $\Omega$, we can now turn into the computation of the marginal and posterior DGN distributions.

\section{Analytical Moments for truncated Gaussian}
\label{appendix:analytical_form}

To lighten the derivation, we introduce extend the $[.]$ indexing operator such that for example for a matrix, $[.]_{-k,.}$ means that all the rows but the $k^{\rm th}$ are taken, and all columns are taken. Also, $[.]_{(k,l),.}$ means that only the $k^{\rm th}$ and $l^{\rm th}$ rows are taken and all the columns.
Let also introduce the following quantities
\begin{align*}
    [F(\ba,\bSigma)]_k=&\phi\left([\ba]_k;0,[\bSigma]_{k,k}\right)\Phi_{[[\ba]_{-k},\infty)}\big(\bmu(k), \bSigma(k)\big)\\
    [G(\ba,\bSigma)]_{k,l}=&\phi\left([\ba]_{(k,l)};0,[\bSigma]_{(k,j),(k,j)}\right)\Phi_{[[\ba]_{-(k,l)},\infty)}\Big(\bmu\big((k,l)\big),\bSigma\big((k,l)\big)\Big)
    \\
    H(\ba,\bSigma)=&G(\ba,\Sigma)+ \diag\left(\frac{\bl \odot F(\bl,\bSigma)-\big(\bSigma\odot G(\bl,\Sigma) \big)\mathbf{1}}{\diag(\bSigma)}\right)
\end{align*}
with $\bmu(u) = [\bSigma]_{-u,u}[\bSigma]_{u,u}^{-1}[\ba]_u$, and  $\bSigma(u)=[\bSigma]_{-u, -u}-[\bSigma]_{-u, u}[\bSigma]_{u,u}^{-1}[\bSigma]_{-u, u}^T$. Thanks to the above form, we can now obtain the integral $e^0_{\omega}(\bSigma)\triangleq \Phi_{\omega}(\mathbf{0},\bSigma)$ and the first two moments of a centered truncated gaussian $\be^1_{\omega}(\bSigma)\triangleq\int_{\omega}\bz \phi(\bz;\mathbf{0},\Sigma)$ and $\bE^2_{\omega}(\bSigma)\triangleq\int_{\omega}\bz\bz^T \phi(\bz;\mathbf{0},\Sigma)$

\begin{cor}
The integral and first two moments of a centered truncated gaussian are given by 
\begin{align}
e^0_{\omega}(\bSigma)=&\sum_{\Delta \in T(\omega)}\sum_{(s,C) \in T(\Delta)} s \Phi_{[\bl(C),\infty)}\left(0,R_c\bSigma R_{c}^T\right)d\bz,\\
\be^1_{\omega}(\bSigma)=&\bSigma\sum_{\Delta \in T(\omega)}\sum_{(s,C) \in T(\Delta)} s R_{C}^TF(\bl_{\omega, c},R_c\bSigma R_{c}^T),\\
\bE^2_{\omega}(\bSigma)=&\bSigma\left(\sum_{\Delta \in T(\omega)}\sum_{(s,C) \in T(\Delta)} s R_{C}^T(H(\bl_{\omega,C},R_c\bSigma R_{c}^T))R_{C}\right)\bSigma +e^0_{\omega}(\bSigma)\bSigma
\end{align}
\end{cor}

To simplify notations let consider the following notation of the posterior (\ref{eq:posterior}) where are incorporate the terms independent of $\bz$ into
\begin{align}
    \alpha_{\omega}(\bx)=\frac{
\phi(\bx;B_{\omega},\bSigma_{\bx}+A_{\omega}\bSigma_{\bz}A_{\omega}^T)}{\sum_{\omega}\phi(\bx;B_{\omega},\bSigma_{\bx}+A_{\omega}\bSigma_{\bz}A_{\omega}^T)\Phi_{\omega}(\bmu_{\omega}(\bx),\bSigma_{\omega})},
\end{align}
leading to $p(\bz|\bx)=\sum_{\omega \in \Omega}\delta_{\omega}(\bz)\alpha_{\omega}(\bx)\phi(\bz;\bmu_{\omega}(\bx),\bSigma_{\omega})$.

\begin{thm}
\label{thm:all_analytical}
The first (per region) moments of the DGN posterior are given by
\begin{align*}
\mathbb{E}_{\bz|\bx}[\Indic_{\bz\in\omega}]&=\alpha_{\omega}(\bx)e^0_{\omega}(\bSigma_{\omega}),\\
\mathbb{E}_{\bz|\bx}[\bz\Indic_{\bz\in\omega}]&=\alpha_{\omega}(\bx)\big(\be^1_{\omega-\bmu_{\omega}(\bx)}(\bSigma_{\omega})+e^0_{\omega-\bmu_{\omega}(\bx)}(\bSigma_{\omega})\bmu_{\omega}(\bx)\big)\\
%
\mathbb{E}_{\bz|\bx}[\bz\bz^T\Indic_{\bz\in\omega}] &= \alpha_{\omega}(\bx)\big(\bE^2_{\omega-\bmu_{\omega}(\bx)}(\bSigma_{\omega})+\be^1_{\omega-\bmu_{\omega}(\bx)}(\bSigma_{\omega})\bmu_{\omega}(\bx)^T\\
&\hspace{2cm}+\bmu_{\omega-\bmu_{\omega}(\bx)}(\bx)\be^1_{\omega-\bmu_{\omega}(\bx)}(\bx)^T+\bmu_{\omega}(\bx)\bmu_{\omega}(\bx)^Te^0_{\omega}(\bx)\big)
\end{align*}
which we denote  $\mathbb{E}_{\bz|\bx}[\Indic_{\bz\in\omega}]\triangleq e^0_{\omega}(\bx)$, $\mathbb{E}_{\bz|\bx}[\bz\Indic_{\bz\in\omega}]\triangleq \be^1_{\omega}(\bx)$ and $\mathbb{E}_{\bz|\bx}[\bz\bz^T\Indic_{\bz\in\omega}] \triangleq \bE^2_{\omega}(\bx)$. (Proof in \ref{proof:all_analytical}.)
\end{thm}

\section{Implementation Details}
\label{sec:implementation}

The Delaunay triangulation needs the $\mathcal{V}$-representation of $\omega$, the vertices which convex hull form the region \cite{grunbaum2013convex}. Given that we have the $\mathcal{H}$-representation, finding the vertices is known as the vertex enumeration problem \cite{ dyer1983complexity}.
To compute the triangulation we use the Python scipy \cite{virtanen2020scipy} implementation which interfaces the C/C++ Qhull implementation \cite{barber1996quickhull}. 
To compute the $\mathcal{H}\mapsto \mathcal{V}$ representation and vice-versa we leverage pycddlib \footnote{\url{https://pypi.org/project/pycddlib/}} which interfaces the C/C++ cddlib library \footnote{\url{https://inf.ethz.ch/personal/fukudak/cdd_home/index.html}} employing the double description method \cite{motzkin1953double}.

\section{Figures}

We demonstrate here additional figures for the posterior and marginal distribution of a DGN.

\begin{figure}[H]
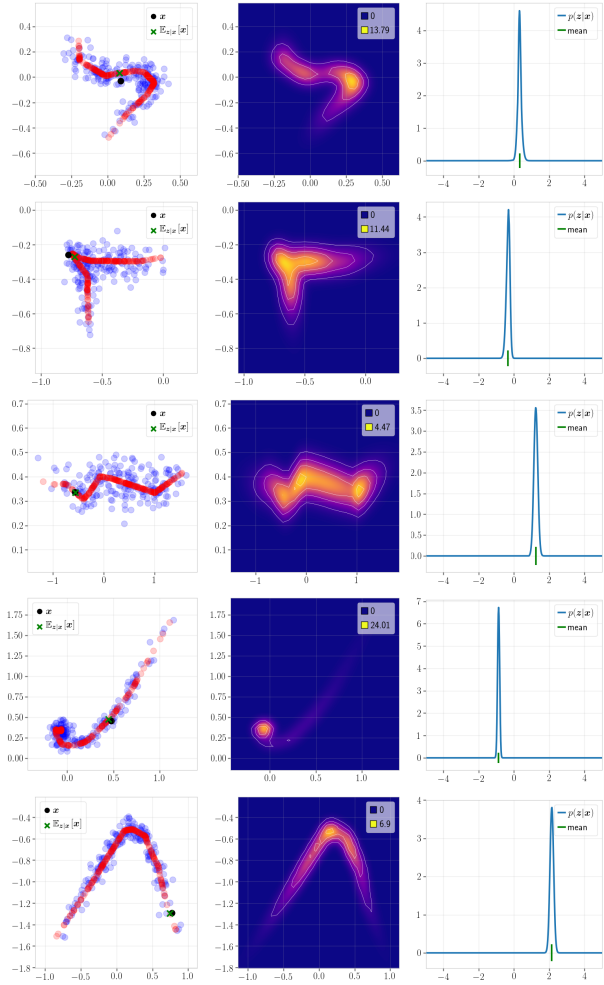

    \centering
    \foreach \n in {146,187,37,79,53}{
    \includegraphics[width=0.14\linewidth]{images/toy/samples_\n_0.png}
    \includegraphics[width=0.14\linewidth]{images/toy/proba_\n_0.png}
    \includegraphics[width=0.14\linewidth]{images/toy/prior_\n_0.png}
    \\}
    \caption{Additional random DGNs with their samples, the posterior and the marginal distributions.}
    \label{fig:my_label}
\end{figure}

\section{Algorithms}

\begin{algorithm}[H]
\SetAlgoLined
\KwData{Starting region $\omega$ and $\bq(\omega)$, initial set ($\Omega$)}
\KwResult{Updated $\Omega$}

\eIf{$\omega \not \in \Omega$}{
$\Omega\gets \Omega\cup \{\omega\}$\;
}{Quit}
$ I=reduce(A^{\rm all}_{\omega},B^{\rm all}_{\omega})$\;
 \For{i $\in$ I}{
  SearchRegion(flip($\bq(\omega),i$), $\Omega$)\;
  }
 \caption{SearchRegion}
 \label{algo:region}
\end{algorithm}

\section{Proofs}
\label{appendix:proofs}

In this section we provide all the proofs for the main paper theoretical claims. In particular we will go through the derivations of the per region posterior first moments and then the derivation of the expectation and maximization steps.

\subsection{Proof of Lemma~\ref{lemma:conditional}}
\label{proof:conditional}
\begin{proof}
The proof consists of expressing the conditional distribution and using the properties of DGN with piecewise affine nonlinearities. We are able to split the distribution into a mixture model as follows:
\begin{align*}
    p(\bx|\bz)=&\frac{1}{(2\pi)^{D/2}\sqrt{|\det{\Sigma_{\bx}}|}}e^{-\frac{1}{2}(\bx - g(\bz))^T\Sigma_{\bx}^{-1}(\bx-g(\bz))}\\
    =&\frac{1}{(2\pi)^{D/2}\sqrt{|\det{\Sigma_{\bx}}|}}e^{-\frac{1}{2}(\bx - \sum_{\omega \in \Omega}\Indic_{\bz \in \omega}(A_{\omega}\bz + B_{\omega}))^T\Sigma_{\bx}^{-1}(\bx-\sum_{\omega \in \Omega}\Indic_{\bz \in \omega}(A_{\omega}\bz + B_{\omega}))}\\
    =&\frac{1}{(2\pi)^{D/2}\sqrt{|\det{\Sigma_{\bx}}|}}e^{-\frac{1}{2}\sum_{\omega \in \Omega}\Indic_{\bz \in \omega}(\bx - (A_{\omega}\bz + B_{\omega}))^T\Sigma_{\bx}^{-1}(\bx-(A_{\omega}\bz + B_{\omega}))}\\
    =&\sum_{\omega \in \Omega}\Indic_{\bz \in \omega}\frac{1}{(2\pi)^{D/2}\sqrt{|\det{\Sigma_{\bx}}|}}e^{-\frac{1}{2}(\bx - (A_{\omega}\bz + B_{\omega}))^T\Sigma_{\bx}^{-1}(\bx-(A_{\omega}\bz + B_{\omega}))}\\
    =&\sum_{\omega \in \Omega}\Indic_{\bz \in \omega}\phi(\bx|A_{\omega}\bz+B_{\omega},\Sigma_{\bx})
\end{align*}
\end{proof}

\subsection{Proof of Proposition~\ref{prop:mppca}}
\label{proof:mppca}

\begin{proof}
This result is direct by noticing that the probability to obtain a specific region slope and bias is the probability that the sampled latent vector lies in the corresponding region. This probability is obtained simply by integrating the latent gaussian distribution on the region. We obtain the result of the proposition.
\end{proof}

\subsection{Proof of Theorem~\ref{thm:px}}
\label{proof:px}

\begin{proof}
For the first part, we simply leverage the known result from linear Gaussian models \cite{roweis1999unifying} stating that

\begin{align*}
    p(\bz|\bx)=&\frac{p(\bx|\bz)p(\bz)}{p(\bx)}\\
    =&\frac{1}{p(\bx)} \frac{e^{-\frac{1}{2}(\bx - g(\bz))^T\Sigma^{-1}_{\bx}(\bx - g(\bz))}}{(2\pi)^{D/2}\sqrt{|\det(\Sigma_{\bx})|}}\frac{e^{-\frac{1}{2}(\bz - \bmu)^T\Sigma^{-1}_{\bz}(\bz - \bmu)}}{(2\pi)^{S/2}\sqrt{|\det(\Sigma_{\bz})|}}\\
    =&\frac{1}{p(\bx)}\Big(\sum_{\omega \in \Omega} \Indic_{\bz \in \omega} \frac{e^{-\frac{1}{2}(\bx - A_{\omega}\bz-B_{\omega})^T\Sigma^{-1}_{\bx}(\bx - A_{\omega}\bz-B_{\omega})}}{(2\pi)^{D/2}\sqrt{|\det(\Sigma_{\bx})|}}\Big)\frac{e^{-\frac{1}{2}(\bz - \bmu)^T\Sigma^{-1}_{\bz}(\bz - \bmu)}}{(2\pi)^{S/2}\sqrt{|\det(\Sigma_{\bz})|}}\\
    =&\frac{1}{p(\bx)}\sum_{\omega \in \Omega} \Indic_{\bz \in \omega} \frac{e^{-\frac{1}{2}(\bx - A_{\omega}\bz-B_{\omega})^T\Sigma^{-1}_{\bx}(\bx - A_{\omega}\bz-B_{\omega})-\frac{1}{2}\bz ^T\Sigma^{-1}_{\bz}\bz }}{(2\pi)^{(S+D)/2}\sqrt{|\det(\Sigma_{\bx})||\det(\Sigma_{\bz})|}}\\
    =&\frac{1}{p(\bx)}\sum_{\omega \in \Omega} \Indic_{\bz \in \omega} \frac{e^{-\frac{1}{2}((\bx-B_{\omega}) - A_{\omega}\bz)^T\Sigma^{-1}_{\bx}((\bx -B_{\omega})- A_{\omega}\bz)-\frac{1}{2}\bz^T\Sigma^{-1}_{\bz}\bz }}{(2\pi)^{(S+D)/2}\sqrt{|\det(\Sigma_{\bx})||\det(\Sigma_{\bz})|}}\\
    =&\frac{1}{p(\bx)}\sum_{\omega \in \Omega} \Indic_{\bz \in \omega} \frac{e^{
    -\frac{1}{2}((A_{\omega}^T\Sigma^{-1}_{\bx}A_{\omega}+\Sigma_{\bz}^{-1})^{-1}A_{\omega}^T\Sigma_{\bx}^{-1}(\bx-B_{\omega}) - \bz)^T(A_{\omega}^T\Sigma^{-1}_{\bx}A_{\omega}+\Sigma_{\bz}^{-1})((A_{\omega}^T\Sigma^{-1}_{\bx}A_{\omega}+\Sigma_{\bz}^{-1})^{-1}A_{\omega}^T\Sigma_{\bx}^{-1}(\bx -B_{\omega})- \bz) }}{(2\pi)^{(S+D)/2}\sqrt{|\det(\Sigma_{\bx})||\det(\Sigma_{\bz})|}}\\
    &\times e^{-\frac{1}{2}((\bx-B_{\omega})^T\Sigma_{\bx}^{-1}(\bx-B_{\omega}))+\frac{1}{2}((\bx-B_{\omega})^T\Sigma_{\bx}^{-1}A_{\omega}(A_{\omega}^T\Sigma^{-1}_{\bx}A_{\omega}+\Sigma_{\bz}^{-1})^{-1}A_{\omega}^T\Sigma_{\bx}^{-1}(\bx -B_{\omega}))}\\
    =&\frac{1}{p(\bx)}\sum_{\omega \in \Omega} \Indic_{\bz \in \omega} \frac{e^{-\frac{1}{2}((A_{\omega}^T\Sigma^{-1}_{\bx}A_{\omega}+\Sigma_{\bz}^{-1})^{-1}A_{\omega}^T\Sigma_{\bx}^{-1}(\bx-B_{\omega}) - \bz)^T(A_{\omega}^T\Sigma^{-1}_{\bx}A_{\omega}+\Sigma_{\bz}^{-1})((A_{\omega}^T\Sigma^{-1}_{\bx}A_{\omega}+\Sigma_{\bz}^{-1})^{-1}A_{\omega}^T\Sigma_{\bx}^{-1}(\bx -B_{\omega})- \bz) }}{(2\pi)^{(S+D)/2}\sqrt{|\det(\Sigma_{\bx})||\det(\Sigma_{\bz})|}}\\
    &\times e^{-\frac{1}{2}((\bx-B_{\omega})^T(\Sigma_{\bx}^{-1}-\Sigma_{\bx}^{-1}A_{\omega}(A_{\omega}^T\Sigma^{-1}_{\bx}A_{\omega}+\Sigma_{\bz}^{-1})^{-1}A_{\omega}^T\Sigma_{\bx}^{-1})(\bx-B_{\omega}))}\\
    =&\frac{1}{p(\bx)}\sum_{\omega \in \Omega} \Indic_{\bz \in \omega} \frac{e^{-\frac{1}{2}((A_{\omega}^T\Sigma^{-1}_{\bx}A_{\omega}+\Sigma_{\bz}^{-1})^{-1}A_{\omega}^T\Sigma_{\bx}^{-1}(\bx-B_{\omega}) - \bz)^T(A_{\omega}^T\Sigma^{-1}_{\bx}A_{\omega}+\Sigma_{\bz}^{-1})((A_{\omega}^T\Sigma^{-1}_{\bx}A_{\omega}+\Sigma_{\bz}^{-1})^{-1}A_{\omega}^T\Sigma_{\bx}^{-1}(\bx -B_{\omega})- \bz) }}{(2\pi)^{(S+D)/2}\sqrt{|\det(\Sigma_{\bx})||\det(\Sigma_{\bz})|}}\\
    &\times e^{-\frac{1}{2}((\bx-B_{\omega})^T(\Sigma_{\bx}+A_{\omega}\Sigma_{\bz}A_{\omega}^T)^{-1}(\bx-B_{\omega}))}\\
    =&\frac{1}{p(\bx)}\sum_{\omega \in \Omega} \Indic_{\bz \in \omega} \frac{e^{-\frac{1}{2}(\bmu_{\omega}(\bx)- \bz)^T\Sigma_{\omega}^{-1}(\bmu_{\omega}(\bx)- \bz) }}{(2\pi)^{(S+D)/2}\sqrt{|\det(\Sigma_{\bx})||\det(\Sigma_{\bz})|}} e^{-\frac{1}{2}((\bx-B_{\omega})^T(\Sigma_{\bx}+A_{\omega}\Sigma_{\bz}A_{\omega}^T)^{-1}(\bx-B_{\omega}))}
\end{align*}
with $\bmu_{\omega}(\bx)=\Sigma_{\omega}A_{\omega}^T\Sigma_{\bx}^{-1}(\bx-B_{\omega}) $
and $\Sigma_{\omega}=(A_{\omega}^T\Sigma^{-1}_{\bx}A_{\omega}+\Sigma_{\bz}^{-1})^{-1}$
as a result it corresponds to a mixture of truncated gaussian, each living on $\omega$. Now we determine the renormalization constant:
\begin{align*}
    p(\bx)=&\int p(\bx|\bz)p(\bz) d\bz\\
    =&\sum_{\omega\in\Omega}\int_{\omega} \Indic_{\bz \in \omega} \frac{e^{-\frac{1}{2}(\bmu_{\omega}(\bx)- \bz)^T\Sigma_{\omega}^{-1}(\bmu_{\omega}(\bx)- \bz) }}{(2\pi)^{(S+D)/2}\sqrt{|\det(\Sigma_{\bx})||\det(\Sigma_{\bz})|}} e^{-\frac{1}{2}((\bx-B_{\omega})^T(\Sigma_{\bx}+A_{\omega}\Sigma_{\bz}A_{\omega}^T)^{-1}(\bx-B_{\omega}))}d\bz\\
    =&\sum_{\omega\in\Omega}\Indic_{\bz \in \omega}\frac{e^{-\frac{1}{2}((\bx-B_{\omega})^T(\Sigma_{\bx}+A_{\omega}\Sigma_{\bz}A_{\omega}^T)^{-1}(\bx-B_{\omega}))}}{(2\pi)^{D/2}\sqrt{|\det(\Sigma_{\bx})||\det(\Sigma_{\bz})|}}\sqrt{\det(\Sigma_{\omega})}\int_{\omega}  \phi(\bz;\bmu_{\omega}(\bx),\Sigma_{\omega})d\bz\\
    =&\sum_{\omega\in\Omega}\Indic_{\bz \in \omega}\frac{e^{-\frac{1}{2}((\bx-B_{\omega})^T(\Sigma_{\bx}+A_{\omega}\Sigma_{\bz}A_{\omega}^T)^{-1}(\bx-B_{\omega}))}}{(2\pi)^{D/2}\sqrt{|\det(\Sigma_{\bx})||\det(\Sigma_{\bz})|}}\sqrt{\det(\Sigma_{\omega})}\Phi_{\omega}(\bmu_{\omega}(\bx),\Sigma_{\omega})\\
    =&\sum_{\omega\in\Omega}\Indic_{\bz \in \omega}\frac{\sqrt{\det(\Sigma_{\bx}+A_{\omega}\Sigma_{\bz}A_{\omega}^T)\det(\Sigma_{\omega})}}{\sqrt{|\det(\Sigma_{\bx})||\det(\Sigma_{\bz})|}}\phi(\bx;B_{\omega},\Sigma_{\bx}+A_{\omega}\Sigma_{\bz}A_{\omega}^T)\Phi_{\omega}(\bmu_{\omega}(\bx),\Sigma_{\omega}),
\end{align*}
now using the Matrix determinant lemma \cite{harville1998matrix} we have that $\det(\Sigma_{\bx}+A_{\omega}\Sigma_{\bz}A_{\omega}^T)=\det(\Sigma_{\bz}^{-1}+A_{\omega}^T\Sigma_{\bx}^{-1}A_{\omega})\det(\Sigma_{\bx})\det(\Sigma_{\bz})$ leading to 
\begin{align*}
p(\bx)=&\sum_{\omega}\phi(\bx;B_{\omega},\Sigma_{\bx}+A_{\omega}\Sigma_{\bz}A_{\omega}^T)\Phi_{\omega}(\bmu_{\omega}(\bx),\Sigma_{\omega}),\\
p(\bz|\bx)=&\sum_{\omega}\delta_{\omega}(\bz)\frac{
\phi(\bx;B_{\omega},\Sigma_{\bx}+A_{\omega}\Sigma_{\bz}A_{\omega}^T)\phi(\bz;\bmu_{\omega}(\bx),\Sigma_{\omega})}{\sum_{\omega}\phi(\bx;B_{\omega},\Sigma_{\bx}+A_{\omega}\Sigma_{\bz}A_{\omega}^T)\Phi_{\omega}(\bmu_{\omega}(\bx),\Sigma_{\omega})}.
\end{align*}
\end{proof}

\subsection{Proof of Lemma~\ref{lemma:zeronoise}}
\label{proof:zeronoise}

The proof will consist of observing that the posterior (prior rewriting) can be expressed as a softmax of a quantity rescaled by the standard deviation.

\begin{proof}
\begin{align*}
    \log(\phi(\bz;\bmu_{\omega}(\bx),\Sigma_{\omega})) = & -\frac{1}{2}(\bx-\bmu_{\omega})^T\Sigma_{\omega}^{-1}(\bx -\bmu_{\omega}(\bx))) -\frac{1}{2}\log(\det(\Sigma_{\omega}))+ cst\\
    = &-\frac{1}{2}(\bz-\Sigma_{\omega}A^T_{\omega}\Sigma^{-1}_{\bx}(A_0\bz_0+B_0-B_{\omega}))^T\Sigma_{\omega}^{-1}(\bz-\Sigma_{\omega}A^T_{\omega}\Sigma^{-1}_{\bx}(A_0\bz_0+B_0-B_{\omega}))^T)\\
    &\hspace{5cm}-\frac{1}{2}\log(\det(\Sigma_{\omega})) + cst\\
    = -\frac{1}{2}(\bz&-(A^T_{\omega}\Sigma_{\bx}^{-1}A_{\omega})^{-1}A^T_{\omega}\Sigma^{-1}_{\bx}(A_0\bz_0+B_0-B_{\omega}))^T\Sigma_{\omega}^{-1}(\bz-(A^T_{\omega}\Sigma_{\bx}^{-1}A_{\omega})^{-1}A^T_{\omega}\Sigma^{-1}_{\bx}(A_0\bz_0+B_0-B_{\omega}))^T)\\
    &\hspace{5cm}-\frac{1}{2}\log(\det(\Sigma_{\omega})) + cst
\end{align*}
where we used the following result to develop $\bmu_{\omega}(\bx)$
\begin{align*}
    \Sigma_{\omega}=(A^T_{\omega}\Sigma^{-1}_{\bx}A_{\omega}+\Sigma_{\bz}^{-1})^{-1}=&(A^T_{\omega}\Sigma^{-1}_{\bx}A_{\omega}+(A^T_{\omega}\Sigma^{-1}_{\bx}A_{\omega})(A^T_{\omega}\Sigma^{-1}_{\bx}A_{\omega})^{-1}\Sigma_{\bz}^{-1})^{-1}\\
    =&(A^T_{\omega}\Sigma^{-1}_{\bx}A_{\omega})^{-1}(I+(A^T_{\omega}\Sigma^{-1}_{\bx}A_{\omega})^{-1}\Sigma_{\bz}^{-1})^{-1}\\
    =&(A^T_{\omega}\Sigma^{-1}_{\bx}A_{\omega})^{-1}\text{ as }(\Sigma_{\bz}A^T_{\omega}\Sigma^{-1}_{\bx}A_{\omega})^{-1} \rightarrow \mathbf{0}.
\end{align*}
if we are in the same region $\omega$ than $\bz_0$ then the above becomes
\begin{align*}
    \argmax_{\bz \in \omega_0}\log(\phi(\bz;\bmu_{\omega}(\bx),\Sigma_{\omega})) 
    = \argmax_{\bz \in \omega_0}-\frac{1}{2}(\bz-\bz_0)^T\Sigma_{\omega}^{-1}(\bz-\bz_0)=\bz_0,
\end{align*}
and since we know that we are in the same region, the argmax $\bz=\bz_0$ lies in this region and thus is the maximum of the posterior.

\end{proof}

\subsection{Proof of Lemma~\ref{lemma:bijection}}
\label{proof:bijection}
\begin{proof}
The sign vectors represent the sign of each pre-activation feature maps. The key here is that when changing the sign, the input passes through the knot of the corresponding activation function of that layer. This implies a change in the region in the DGN input space. In fact, without degenerate weights and with nonzero activation functions, a change in any dimension of the sign vector (used to form the per region slope and bias) impact a change in the affine mapping used to map inputs $\bz$ to outputs $\bx$. As such, whenever a sign changes, the affine mapping changes, leading to a change of region in the DGN input space. As the sign vector is formed from the DGN input space, and we restrict ourselves to the image of this mapping, there does not exist a sign pattern/configuration not reachable by the DGN (otherwise it would not be in the image of this mapping). Now for the other inclusion, recall that a change in region and thus in per region affine mapping can only occur with a change of pre-activation sign pattern.
\end{proof}

\subsection{Proof of Corollary\ref{cor:H_rep}}
\label{proof:H_rep}

\begin{proof}
From the above result, it is clear that the preactivation roots define the boundaries of the regions. Obtaining the hyperplane representation of the region thus simply consists of reexpressing this statement with the explicit pre-activation hyperplanes for all the layers and units, the intersection between layers coming from the subdivision. For additional details please see \cite{balestriero2019geometry}.
\end{proof}

\subsection{Proof of Lemma~\ref{lemma:integral_decomposition}}
\label{proof:intergral_decomposition}
\begin{proof}
The proof consists of rearranging the terms from the inclusion-exclusion formula as in
\begin{align*}
    \sum_{J \subseteq \{1,\dots,F\}, J\not = \emptyset}&(-1)^{|J|+1}\left(\cap_{j\in J}A_j\right)=\cup_{i}A_i\\
    (-1)^{F+1}S+\sum_{J \subseteq \{1,\dots,F\}, J\not = \emptyset,|J|<F}&(-1)^{|J|+1}\left(\cap_{j\in J}A_j\right)=\cup_{i}A_i\\
    (-1)^{F+1}S&=\cup_{i}A_i-\sum_{J \subseteq \{1,\dots,F\}, J\not = \emptyset,|J|<F}(-1)^{|J|+1}\left(\cap_{j\in J}A_j\right)\\
    S&=(-1)^{F+1}\cup_{i}A_i-(-1)^{F+1}\sum_{J \subseteq \{1,\dots,F\}, J\not = \emptyset,|J|<F}(-1)^{|J|+1}\left(\cap_{j\in J}A_j\right)\\
    S&=(-1)^{F+1}\cup_{i}A_i+\sum_{J \subseteq \{1,\dots,F\}, J\not = \emptyset,|J|<F}(-1)^{|J|+1+F}\left(\cap_{j\in J}A_j\right)
\end{align*}
then by application of Chasles rule \cite{hirsch2012elements}, the integral domain can be decomposed into the signed sum of per cone integration. Finally, a simplex in dimension $S$ has $S+1$ faces, making $F=S+1$ and leading to the desired result.
\end{proof}

\subsection{Proof of Moments}

\begin{lemma}
The first moments of Gaussian integration on an open rectangle defined by its lower limits $\ba$ is given by
\begin{align}
    \int_{\ba}^{\infty}\bz \phi(\mathbf{0},\bSigma)d\bz=&\bSigma F(\ba),\\
    \int_{\ba}^{\infty}\bz\bz^T \phi(\mathbf{0},\bSigma)d\bz=&\Phi_{[\ba,\infty)}(\mathbf{0},\bSigma) \bSigma+\bSigma \left(G(\ba)+\frac{\ba \odot F(\ba)-\big(\bSigma \odot G(\ba) \big)\mathbf{1}}{\diag(\bSigma)}\right)\bSigma.\label{eq:phi}
\end{align}
where the division is performed elementwise. 
\end{lemma}

\begin{proof}
First moment:
\begin{align*}
\int_{\omega} \bz \phi(\bx;\mathbf{0},\Sigma)d\bz =&  \int_{\omega} \bz \frac{e^{-\frac{1}{2}\bz^T\Sigma^{-1}\bz}}{(2\pi)^{K/2}|\det(\Sigma)|^{1/2}}d\bz\\
=& \sum_{\Delta \in S(\omega)}\sum_{(s,C) \in T(\Delta)} s \int_{C} \bz \frac{e^{-\frac{1}{2}(R_C\bz)^T(R_C^T)^{-1}\Sigma_{\omega}^{-1}R^{-1}_CR_C\bz}}{(2\pi)^{K/2}|\det(\Sigma_{\omega})|^{1/2}}d\bz\\
=&  \sum_{\Delta \in S(\omega)}\sum_{(s,C) \in T(\Delta)} s\int_{\bl(C)} R^{-1}\bu \frac{e^{-\frac{1}{2}\bu^T(R_C\Sigma_{\omega}R_C^T)^{-1}\bu}}{(2\pi)^{K/2}|\det(R_C)||\det(\Sigma_{\omega})|^{1/2}}d\bu\\
=& \sum_{\Delta \in S(\omega)}\sum_{(s,C) \in T(\Delta)} sR^{-1}_C\int_{\bl(C)} \bu \phi(\bu;\mathbf{0},R_C\Sigma_{\omega}R_C^T)d\bu\\
=& \sum_{\Delta \in S(\omega)}\sum_{(s,C) \in T(\Delta)} sR^{-1}_C(R_C\Sigma_{\omega}R_C^TF(\bl(C))\\
=& \Sigma_{\omega}\sum_{\Delta \in S(\omega)}\sum_{(s,C) \in T(\Delta)} sR_C^TF(\bl(C))
\end{align*}

Second moment
\begin{align*}
\int_{\omega} \bz\bz^T \phi(\bx;\mathbf{0},\Sigma)d\bz =&  \int_{\omega} \bz\bz^T \frac{e^{-\frac{1}{2}\bz^T\Sigma^{-1}\bz}}{(2\pi)^{K/2}|\det(\Sigma)|^{1/2}}d\bz\\
=& \sum_{\Delta \in S(\omega)}\sum_{(s,C) \in T(\Delta)} s \int_{C} \bz\bz^T \frac{e^{-\frac{1}{2}(R_C\by)^T(R_C^T)^{-1}\Sigma_{\omega}^{-1}R^{-1}_CR_C\by}}{(2\pi)^{K/2}|\det(\Sigma_{\omega})|^{1/2}}d\bz\\
=&  \sum_{\Delta \in S(\omega)}\sum_{(s,C) \in T(\Delta)} s\int_{\bl(C)} R_C^{-1}\bu\bu^T(R_C^{-1})^T \frac{e^{-\frac{1}{2}\bu^T(R_C\Sigma_{\omega}R_C^T)^{-1}\bu}}{(2\pi)^{K/2}|\det(R_C)||\det(\Sigma_{\omega})|^{1/2}}d\bu\\
=& \sum_{\Delta \in S(\omega)}\sum_{(s,C) \in T(\Delta)} sR^{-1}_C\int_{\bl(C)} \bu\bu^T \phi(\bu;\mathbf{0},R_C\Sigma_{\omega}R_C^T)d\bu(R^{-1}_C)^T\\
=& \sum_{\Delta \in S(\omega-\bmu_{\omega}(\bx))}\sum_{(s,C) \in T(\Delta)} sR^{-1}_C\Big[
\Phi_{[\bl(C),\infty)}(\mathbf{0},R_C\Sigma_{\omega}R_C^T) R_C\Sigma_{\omega}R_C^T\\
&+R_C\Sigma_{\omega}R_C^T \left(\frac{\bl(C) \odot F(\bl(C))+\big(R_C\Sigma_{\omega}R_C^T \odot G(\bl(C)) \big)\mathbf{1}}{\diag(R_C\Sigma_{\omega}R_C^T)}\right)(R_C\Sigma_{\omega}R_C^T)^T
\Big](R^{-1}_C)^T\\
=& \sum_{\Delta \in S(\omega-\bmu_{\omega}(\bx))}\sum_{(s,C) \in T(\Delta)} s\Big[
\Phi_{[\bl(C),\infty)}(\mathbf{0},R_C\Sigma_{\omega}R_C^T)\Sigma_{\omega}\\
&+\Sigma_{\omega}R_C^T \left(\frac{\bl(C) \odot F(\bl(C))+\big(R_C\Sigma_{\omega}R_C^T \odot G(\bl(C)) \big)\mathbf{1}}{\diag(R_C\Sigma_{\omega}R_C^T)}\right)R_C\Sigma_{\omega}
\Big]\\
&\hspace{-2cm}= e^0_{\omega-\bmu_{\omega}(\bx)}\Sigma_{\omega}+\Sigma_{\omega}\big[\sum_{\Delta \in S(\omega-\bmu_{\omega}(\bx))}\sum_{(s,C) \in T(\Delta)} sR_C^T \left(\frac{\bl(C) \odot F(\bl(C))+\big(R_C\Sigma_{\omega}R_C^T \odot G(\bl(C)) \big)\mathbf{1}}{\diag(R_C\Sigma_{\omega}R_C^T)}\right)R_C\big]\Sigma_{\omega}
\end{align*}
\end{proof}

\subsection{Proof of Theorem~\ref{thm:all_analytical}}
\label{proof:all_analytical}

\begin{proof}
Constant:$R_{C}=\begin{pmatrix}C^T\\H^T\Sigma_{\omega}^{-1}\end{pmatrix}$
\begin{align*}
\int_{\omega} p(\bz|\bx)d\bz 
= \alpha_{\omega}(\bx)\int_{\omega}\phi(\bz;\bmu_{\omega}(\bx),\Sigma_{\omega})d\bz
= \alpha_{\omega}(\bx)\int_{\omega-\bmu_{\omega}(\bx)}\phi(\bz;\mathbf{0},\Sigma_{\omega})d\bz
=\alpha_{\omega}(\bx)e^0_{\omega-\bmu_{\omega}(\bx)}
\end{align*}

First moment:
\begin{align*}
\int_{\omega} \bz p(\bz|\bx)d\bz =&  \alpha_{\omega}(\bx)\int_{\omega} \bz \frac{e^{-\frac{1}{2}(\bz-\bmu_{\omega}(\bx))^T\Sigma_{\omega}^{-1}(\bz-\bmu_{\omega}(\bx))}}{(2\pi)^{K/2}|\det(\Sigma_{\omega})|^{1/2}}d\bz\\
=&  \alpha_{\omega}(\bx)\int_{\omega-\bmu_{\omega}(\bx)} (\by+\bmu_{\omega}(\bx)) \frac{e^{-\frac{1}{2}\by^T\Sigma_{\omega}^{-1}\by}}{(2\pi)^{K/2}|\det(\Sigma_{\omega})|^{1/2}}d\bz\\
=&  \alpha_{\omega}(\bx)\left(\be^1_{\omega-\bmu_{\omega}(\bx)}+e^0_{\omega-\bmu_{\omega}(\bx)}\bmu_{\omega}(\bx)\right)
\end{align*}

Second moment:
\begin{align*}
\int \bz\bz^T p(\bz|\bx)d\bz =&  \alpha_{\omega}(\bx)\int_{\omega}\bz\bz^T \phi(\bz;\bmu_{\omega}(\bx),\Sigma_{\omega})d\bz\\
=& \alpha_{\omega}(\bx)\int_{\omega-\mu_{\omega}(\bx)}(\by+\mu_{\omega}(\bx))(\by+\mu_{\omega}(\bx))^T \phi(\bz;\mathbf{0},\Sigma_{\omega})d\bz\\
=& \alpha_{\omega}(\bx)\left( \bE^2+\bmu_{\omega}(\bx)\be^1_{\omega}(\bSigma_{\omega})^T+\be^1_{\omega}(\bSigma_{\omega})\mu_{\omega}(\bx)^T+\mu_{\omega}(\bx)\mu_{\omega}(\bx)^Te^0_{\omega-\mu_{\omega}(\bx)}(\bSigma_{\omega})\right)
\end{align*}

\end{proof}

\section{Proof of EM-step}
\label{proof:EM}

We now derive the expectation maximization steps for a piecewise affine and continuous DGN.

\subsection{E-step derivation}
\label{proof:Estep}
\begin{align}
    E_{\bz|\bx}[(A_{\omega}\bz+B_{\omega})\Indic_{\omega}]&=A m^1_{\omega}+Be^0_{\omega}\\
    E_{\bz|\bx}[\bz^TA_{\omega}^TA_{\omega}\bz\Indic_{\omega}]&=\trace(A^T_{\omega}A_{\omega}m^2)
\end{align}
\begin{align*}
    E_{Z|X}\left[\log \left( p_{X|Z}(\bx|\bz)p_{Z}(\bz) \right)\right]=&
    E_{Z|X}\left[\log \left( \frac{e^{-\frac{1}{2}(\bx-g(\bz))^T\Sigma_{\bx}^{-1}(\bx-g(\bz))}}{(2\pi)^{D/2}\sqrt{|\det(\bSigma_{\bx})|}} \frac{e^{-\frac{1}{2}\bz^T\Sigma_{\bz}^{-1}\bx}}{(2\pi)^{S/2}\sqrt{|\det(\bSigma_{\bz})|}} \right)\right]\\
    =-\log\Big( (2\pi)^{(S+D)/2}&\sqrt{|\det(\bSigma_{\bz})|}\sqrt{|\det(\bSigma_{\bx})|}) \Big)-\frac{1}{2}E_{Z|X}\bigg[(\bx-g(\bz))^T\bSigma_{\bx}^{-1}(\bx-g(\bz))+\bz^T\bSigma_{\bz}^{-1}\bz\bigg]\\
    =-\log\Big( (2\pi)^{(S+D)/2}&\sqrt{|\det(\bSigma_{\bz})|}\sqrt{|\det(\bSigma_{\bx})|}) \Big)\\
    &-\frac{1}{2}\bigg(\bx^T\Sigma_{\bx}^{-1}\bx+E_{Z|X}\bigg[-2 \bx^T\bSigma_{\bx}^{-1}g(\bz) +g(\bz)^T\bSigma_{\bx}^{-1}g(\bz)+\bz^T\bSigma_{\bz}^{-1}\bz\bigg]\bigg)\\
    =-\log\Big( (2\pi)^{(S+D)/2}&\sqrt{|\det(\bSigma_{\bz})|}\sqrt{|\det(\bSigma_{\bx})|}) \Big)-\frac{1}{2}\bigg(\bx^T\Sigma_{\bx}^{-1}\bx+\trace(E_{Z|X}[\bz\bz^T\bSigma_{\bz}^{-1}])\\
    &\hspace{4cm}+E_{Z|X}\bigg[-2 \bx^T\bSigma_{\bx}^{-1}g(\bz) +g(\bz)^T\bSigma_{\bx}^{-1}g(\bz)\bigg]\bigg)\\
    =-\log\Big( (2\pi)^{(S+D)/2}&\sqrt{|\det(\bSigma_{\bz})|}\sqrt{|\det(\bSigma_{\bx})|}) \Big)-\frac{1}{2}\bigg(\bx^T\bSigma^{-1}_{\bx}\bx-2 \bx^T\bSigma_{\bx}^{-1}\left(\sum_{\omega}\bA_{\omega}\be^1_{\omega}(\bx)+\bb_{\omega}e^0_{\omega}(\bx)\right)\\
    +&\sum_{\omega} e^0_{\omega}\bb_{\omega}^T\bSigma_{\bx}^{-1}\bb_{\omega}+\trace(\bA_{\omega}^T\Sigma_{\bx}^{-1}\bA_{\omega}\bE^2_{\omega}(\bx))+2(\bA_{\omega}\bm_{\omega}^1(\bx))^T\bSigma_{\bx}^{-1}\bb_{\omega}+\trace(\Sigma_{\bz}^{-1}\bE^2(\bx)) \bigg)
\end{align*}

\subsection{Proof of M step}

Let first introduce some notations:
\begin{gather*}
    \bA^{L\rightarrow i}_{\omega}\triangleq (\bA^{L\rightarrow i}_{\omega})^T(\text{back-propagation matrix to layer $i$}),\\
    r_{\omega}^{\ell}(\bx) \triangleq \left(\bx e^0_{\omega}(\bx)-\left(\bA_{\omega}\be^1_{\omega}(\bx)+\sum_{i\not = \ell}m^{0}_{\omega}(\bx)\bA^{i + 1\rightarrow L}_{\omega}\bQ^{i}_{\omega}\bv^{i}\right)\right)\;\;(\text{expected residual without  $\bv^\ell$})\\
    \hat{\bz}_{\omega}^{\ell}(\bx)\triangleq \bQ_{\omega}^{\ell-1}\left(\bA_{\omega}^{1\rightarrow \ell-1} m^1_{\omega}(\bx)+\bb_{\omega}^{1\rightarrow \ell-1}e^0_{\omega}\right)\;\;(\text{expected feature map of layer $\ell$})
\end{gather*}
we can now provide the analytical forms of the M step for each of the learnable parameters:
\begin{gather}
    \bSigma_{\bx}^*\hspace{-0.1cm}=\frac{1}{N}\sum_{\bx}\left(\bx\bx^T\hspace{-0.13cm}+\hspace{-0.1cm}\sum_{\omega}\bb_{\omega} \left(\bb_{\omega}m_{\omega}^0(\bx)+2\bA_{\omega}\be^1_{\omega}(\bx)\right)^T\hspace{-0.13cm}-2 \bx(\hat{\bz}_{\omega}^{L}(\bx))^T\hspace{-0.13cm}+\hspace{-0.05cm}\bA_{\omega}\bE^2_{\omega}(\bx)\bA^T_{\omega}\right),\\
    {\bv^{\ell}}^* =\left(\sum_{\bx}\sum_{\omega}\bQ^{\ell}_{\omega}\bA^{L\rightarrow \ell+1}_{\omega}\bSigma_{\bx}^{-1}\bA^{\ell+1\rightarrow L}_{\omega}\bQ^{\ell}_{\omega}\right)^{-1}\hspace{-0.13cm}\left(\sum_{\bx}\sum_{\omega \in \Omega} \underbrace{\bQ^{\ell}_{\omega}\bA^{L\rightarrow \ell+1}_{\omega}\bSigma_{\bx}^{-1}  r^{\ell}_{\omega}(\bx)}_{\text{residual back-propagated to layer $\ell$}}\right),\\
    \vect({\bW^{\ell}}^{*})= U_{\omega}^{-1}\vect\bigg(\hspace{-0.05cm}\sum_{\bx}\sum_{\omega}\underbrace{\bQ_{\omega}^{\ell}\bA^{L\rightarrow \ell + 1}\bSigma_{\bx}^{-1}\hspace{-0.05cm}\left(\bx-\sum_{i=\ell}^{L}\bA_{\omega}^{i + 1 \rightarrow L}\bQ^{i}_{\omega}\bv^{i}   \right)}_{\text{residual back-propagated to layer $\ell$}}(\hat{\bz}^{\ell}_{\omega}(\bx))^T\hspace{-0.05cm}\bigg),
\end{gather}

we provide detailed derivations below.

\subsubsection{Update of the bias parameter}

Recall from (\ref{eq:region_parameters}) that 
$
\bb_{\omega} = \bv^{L}+\sum_{i=1}^{L-1}\bW^{L}\bQ_{\omega}^{L-1}\bW^{L-1}\dots \bQ_{\omega}^{i}\bv^{i},\l
$
we can thus rewrite the loss as

\begin{align*}
    L(\bv^{\ell})=&\loss\\
    =&-\frac{1}{2}\bigg(-2 \bx^T\bSigma_{\bx}^{-1}\left(\sum_{\omega}\bb_{\omega}e^0_{\omega}(\bx)\right)
    +\sum_{\omega} e^0_{\omega}\bb_{\omega}^T\bSigma^{-1}_{\bx}\bb_{\omega}+2\sum_{\omega} (\bA_{\omega}\bm_{\omega}^1(\bx))^T\bSigma_{\bx}^{-1} \bb_{\omega}\Big)+cst\\
    =&-\frac{1}{2}\sum_{\omega}\bigg(-2 \bx^T\bSigma_{\bx}^{-1}\left(\bA^{\ell+1\rightarrow L}_{\omega}\bQ^{\ell}_{\omega}\bv^{\ell}e^0_{\omega}(\bx)\right)
    + e^0_{\omega}(\bA^{\ell+1\rightarrow L}_{\omega}\bQ^{\ell}_{\omega}\bv^{\ell})^T\bSigma_{\bx}^{-1}(\bA^{\ell+1\rightarrow L}_{\omega}\bQ^{\ell}_{\omega}\bv^{\ell})\\
    &+2e^0_{\omega}(\bx)(\sum_{i\not = \ell}\bA^{i+1\rightarrow L}_{\omega}\bQ^{i}_{\omega}\bv^{i})^T \bSigma_{\bx}^{-1}(\bA^{\ell+1\rightarrow L}_{\omega}\bQ^{\ell}_{\omega}\bv^{\ell})
    +2 ((m_{\omega}^1(\bx))^T(\bA_{\omega})^T\bSigma_{\bx}^{-1}\bA^{\ell+1\rightarrow L}_{\omega}\bQ^{\ell}_{\omega}\bv^{\ell}\Big)+cst\\
    =&-\frac{1}{2}\sum_{\omega}\bigg( e^0_{\omega}(\bA^{\ell+1\rightarrow L}_{\omega}\bQ^{\ell}_{\omega}\bv^{\ell})^T\bSigma_{\bx}^{-1}(\bA^{\ell+1\rightarrow L}_{\omega}\bQ^{\ell}_{\omega}\bv^{\ell})&&(A)\\
    &+2(e^0_{\omega}(\bx)(\sum_{i\not = \ell}\bA^{i+1\rightarrow L}_{\omega}\bQ^{i}_{\omega}\bv^{i}-\bx)+\bA_{\omega}\be^1_{\omega}(\bx))^T \bSigma_{\bx}^{-1}(\bA^{\ell+1\rightarrow L}_{\omega}\bQ^{\ell}_{\omega}\bv^{\ell})\Big)+cst&&(B)\\
    \implies \partial L(\bv^{\ell})=& -\frac{1}{2}\sum_{\omega}\Bigg[-e^0_{\omega}(\bx)2\bQ^{\ell}_{\omega}\bA^{L\rightarrow \ell+1}_{\omega}\bSigma_{\bx}^{-1}\bA^{\ell+1\rightarrow L}_{\omega}\bQ^{\ell}_{\omega}\bv^{\ell}\\
    &+2\left(\bA^{\ell+1\rightarrow L}_{\omega}\bQ^{\ell}_{\omega}\right)^T\bSigma_{\bx}^{-1}\left(e^0_{\omega}(\bx)\left(\sum_{i\not = \ell}\bA^{i+1\rightarrow L}_{\omega}\bQ^{i}_{\omega}\bv^{i}-\bx\right)+\bA_{\omega}\be^1_{\omega}(\bx)\right)\Bigg]\\
    \implies \bv^{\ell}=&\left(\sum_{\bx}\sum_{\omega}e^0_{\omega}(\bx)\bQ^{\ell}_{\omega}\bA^{L\rightarrow \ell+1}_{\omega}\bSigma_{\bx}^{-1}\bA^{\ell+1\rightarrow L}_{\omega}\bQ^{\ell}_{\omega}\right)^{-1}\\
    &\times\sum_{\bx}\sum_{\omega \in \Omega} \bQ^{\ell}_{\omega}\bA^{L\rightarrow \ell+1}_{\omega}\bSigma_{\bx}^{-1} \left( \bx e^0_{\omega}(\bx)-\left(\bA_{\omega}m^1_{\omega}(\bx)+\sum_{i\not = \ell}m^{0}_{\omega}(\bx)\bA^{i + 1\rightarrow L}_{\omega}\bQ^{i}_{\omega}\bv^{i}\right)\right)
\end{align*}
as

\begin{align*}
    (A)=&e^0_{\omega}(\bx)(\bA^{\ell+1\rightarrow L}_{\omega}\bQ^{\ell}_{\omega}\bv^{\ell})^T\bSigma_{\bx}^{-1}(\bA^{\ell+1\rightarrow L}_{\omega}\bQ^{\ell}_{\omega}\bv^{\ell})\\
    \implies \partial (A) =& e^0_{\omega}(\bx)2\bQ^{\ell}_{\omega}\bA^{L\rightarrow \ell+1}_{\omega}\bSigma_{\bx}^{-1}\bA^{\ell+1\rightarrow L}_{\omega}\bQ^{\ell}_{\omega}\bv^{\ell}\\
(B)=&2\left[\left(e^0_{\omega}(\bx)\left(\sum_{i\not = \ell}\bA^{i+1\rightarrow L}_{\omega}\bQ^{i}_{\omega}\bv^{i}-\bx\right)+\bA_{\omega}\be^1_{\omega}(\bx)\right)^T \bSigma_{\bx}^{-1}(\bA^{\ell+1\rightarrow L}_{\omega}\bQ^{\ell}_{\omega}\bv^{\ell})\right]+cst \\
\implies \partial (B) =&  \left(\bA^{\ell+1\rightarrow L}_{\omega}\bQ^{\ell}_{\omega}\right)^T\bSigma_{\bx}^{-1}\left(e^0_{\omega}(\bx)\left(\sum_{i\not = \ell}\bA^{i+1\rightarrow L}_{\omega}\bQ^{i}_{\omega}\bv^{i}-\bx\right)+\bA_{\omega}\be^1_{\omega}(\bx)\right)
\end{align*}

\subsubsection{Update of the slope parameter}

We can thus rewrite the loss as 

\begin{align*}
    L(\bv^{\ell})=&\loss\\
    =&
    \bx^T\bSigma_{\bx}^{-1}\left(\sum_{\omega}\bA_{\omega}\be^1_{\omega}(\bx)+\bb_{\omega}e^0_{\omega}(\bx)\right)
    -\frac{1}{2}\sum_{\omega} e^0_{\omega}\bb_{\omega}^T\bSigma^{-1}_{\bx}\bb_{\omega}  -\frac{1}{2}\sum_{\omega}\trace(\bA_{\omega}^T\bSigma_{\bx}^{-1}\bA_{\omega}\bE^2_{\omega}(\bx))\\
    &-\sum_{\omega}(\bA_{\omega}m_{\omega}^1(\bx))^T\bSigma_{\bx}^{-1} \bb_{\omega}
\end{align*}
Notice that we can rewrite $\bb_{\omega}=\bA^{\ell+1\rightarrow L}\bQ_{\omega}^{\ell}\bW^{\ell}\bQ_{\omega}^{\ell-1}\bb_{\omega}^{1\rightarrow \ell-1} + \sum_{i=\ell}^{L}\bA_{\omega}^{i + 1 \rightarrow L}\bQ^{i}_{\omega}\bv^{i}$ and $\bA_{\omega} = \bA^{\ell+1\rightarrow L}_{\omega}\bQ^{\ell}_{\omega}\bW^{\ell}\bQ^{\ell-1}_{\omega}\bA^{1\rightarrow \ell -1}$ and thus we obtain:
\begin{align*}
    L(\bv^{\ell})=&\sum_{\omega}
    \bx^T\bSigma_{\bx}^{-1}\bA^{\ell+1\rightarrow L}\bQ_{\omega}^{\ell}\bW^{\ell}\bQ_{\omega}^{\ell-1}\Big(\bA^{1\rightarrow \ell-1} \be^1_{\omega}(\bx)+\bb^{1\rightarrow \ell-1}_{\omega}e^0_{\omega}(\bx)\Big)\\
    &-\frac{1}{2}\sum_{\omega} e^0_{\omega}(\brewriteA)^T\bSigma^{-1}_{\bx}(\brewriteA)\\
    &-\sum_{\omega} e^0_{\omega}(\brewriteA)^T\bSigma^{-1}_{\bx}(\brewriteb)\\ 
    &-\frac{1}{2}\sum_{\omega}\trace((\Arewrite)^T\bSigma_{\bx}^{-1}(\Arewrite)\bE^2_{\omega}(\bx))\\
    &-\sum_{\omega} (\Arewrite m_{\omega}^1(\bx))^T\bSigma_{\bx}^{-1} (\brewriteA)\\
    &-\sum_{\omega} (\Arewrite m_{\omega}^1(\bx))^T\bSigma_{\bx}^{-1} (\brewriteb)+cst\\
    =&\sum_{\omega}
    \bx^T\bSigma_{\bx}^{-1}\bA^{\ell+1\rightarrow L}\bQ_{\omega}^{\ell}\bW^{\ell}\bQ_{\omega}^{\ell-1}\Big(\bA^{1\rightarrow \ell-1} \be^1_{\omega}(\bx)+\bb^{1\rightarrow \ell-1}_{\omega}e^0_{\omega}(\bx)\Big)&&(A)\\
    &-\frac{1}{2}\sum_{\omega} e^0_{\omega}(\brewriteA)^T\bSigma^{-1}_{\bx}(\brewriteA)&&(B)\\
    &-\sum_{\omega} \Big(\bA^{\ell+1\rightarrow L}\bQ_{\omega}^{\ell}\bW^{\ell}\bQ_{\omega}^{\ell-1}\big(\bA^{1\rightarrow \ell-1}\be^1_{\omega}(\bx)+\bb_{\omega}^{1\rightarrow \ell-1}e^0_{\omega}(\bx)\big)\Big)^T\bSigma^{-1}_{\bx}(\brewriteb)&&(C)\\ 
    &-\frac{1}{2}\sum_{\omega}\trace((\Arewrite)^T\bSigma_{\bx}^{-1}(\Arewrite)\bE^2_{\omega}(\bx))&&(D)\\
    &-\sum_{\omega} (\Arewrite m_{\omega}^1(\bx))^T\bSigma_{\bx}^{-1} (\brewriteA)+cst&&(E)
\end{align*}

\begin{align*}
    A =&\sum_{\omega}
    \bx^T\bSigma_{\bx}^{-1}\bA^{\ell+1\rightarrow L}\bQ_{\omega}^{\ell}\bW^{\ell}\bQ_{\omega}^{\ell-1}\Big(\bA^{1\rightarrow \ell-1} \be^1_{\omega}(\bx)+\bb^{1\rightarrow \ell-1}_{\omega}e^0_{\omega}(\bx)\Big)\\
    \implies \partial A =& \sum_{\omega}\bQ^{\ell}_{\omega}\bA^{L\rightarrow \ell+1}_{\omega}\bSigma_{\bx}^{-1}\bx\left(\bQ^{\ell-1}_{\omega}(\bA^{1\rightarrow \ell -1}\be^1_{\omega}(\bx)+\bb^{1\rightarrow \ell-1}e^0_{\omega}(\bx))\right)^T\\
\end{align*}

\begin{align*}
    B=&-\frac{1}{2}\sum_{\omega}e^0_{\omega}(\bx)\left(\bA^{\ell+1\rightarrow L}\bQ_{\omega}^{\ell}\bW^{\ell}\bQ_{\omega}^{\ell-1}\bb_{\omega}^{1\rightarrow \ell-1}\right)^T\bSigma_{\bx}^{-1}\left(\bA^{\ell+1\rightarrow L}\bQ_{\omega}^{\ell}\bW^{\ell}\bQ_{\omega}^{\ell-1}\bb_{\omega}^{1\rightarrow \ell-1}\right)\\
    =&-\frac{1}{2}\sum_{\omega}e^0_{\omega}(\bx)(\bQ_{\omega}^{\ell-1}\bb_{\omega}^{1\rightarrow \ell-1})^T(\bW^{\ell})^T(\bA^{\ell+1\rightarrow L}\bQ_{\omega}^{\ell})^T\bSigma_{\bx}^{-1}\left(\bA^{\ell+1\rightarrow L}\bQ_{\omega}^{\ell}\bW^{\ell}\bQ_{\omega}^{\ell-1}\bb_{\omega}^{1\rightarrow \ell-1}\right)\\
    =&-\frac{1}{2}\sum_{\omega}e^0_{\omega}(\bx)\trace \left((\bW^{\ell})^T(\bA^{\ell+1\rightarrow L}\bQ_{\omega}^{\ell})^T\bSigma_{\bx}^{-1}\bA^{\ell+1\rightarrow L}\bQ_{\omega}^{\ell}\bW^{\ell}\bQ_{\omega}^{\ell-1}\bb_{\omega}^{1\rightarrow \ell-1}(\bQ_{\omega}^{\ell-1}\bb_{\omega}^{1\rightarrow \ell-1})^T\right)\\
    &\implies \partial B =-\sum_{\omega} e^0_{\omega}(\bx)\bQ_{\omega}^{\ell}\bA^{L\rightarrow \ell+1}\bSigma_{\bx}^{-1}\bA^{\ell+1\rightarrow L}\bQ_{\omega}^{\ell}\bW^{\ell}\bQ_{\omega}^{\ell-1}\bb_{\omega}^{1\rightarrow \ell-1}(\bb_{\omega}^{1\rightarrow \ell-1})^T\bQ_{\omega}^{\ell-1}
\end{align*}

\begin{align*}
    C =& -\sum_{\omega} \Big(\bA^{\ell+1\rightarrow L}\bQ_{\omega}^{\ell}\bW^{\ell}\bQ_{\omega}^{\ell-1}\big(\bA^{1\rightarrow \ell-1}\be^1_{\omega}(\bx)+\bb_{\omega}^{1\rightarrow \ell-1}e^0_{\omega}(\bx)\big)\Big)^T\bSigma^{-1}_{\bx}(\brewriteb)\\
    =& -\sum_{\omega} (\bQ_{\omega}^{\ell-1}\big(\bA^{1\rightarrow \ell-1}\be^1_{\omega}(\bx)+\bb_{\omega}^{1\rightarrow \ell-1}e^0_{\omega}(\bx)\big)^T(\bW^{\ell})^T(\bA^{\ell+1\rightarrow L}\bQ_{\omega}^{\ell})^T\bSigma^{-1}_{\bx}(\brewriteb)\\
    \implies \partial C &=-\sum_{\omega} \bQ_{\omega}^{\ell}\bA^{L\rightarrow \ell+1}\bSigma^{-1}_{\bx}(\brewriteb)(\bQ_{\omega}^{\ell-1}\big(\bA^{1\rightarrow \ell-1}\be^1_{\omega}(\bx)+\bb_{\omega}^{1\rightarrow \ell-1}e^0_{\omega}(\bx)\big)^T
\end{align*}

\begin{align*}
    D=&-\frac{1}{2}\sum_{\omega}\trace((\bA^{\ell+1\rightarrow L}_{\omega}\bQ^{\ell}_{\omega}\bW^{\ell}\bQ^{\ell-1}_{\omega}\bA^{1\rightarrow \ell -1})^T\bSigma_{\bx}^{-1}\bA^{\ell+1\rightarrow L}_{\omega}\bQ^{\ell}_{\omega}\bW^{\ell}\bQ^{\ell-1}_{\omega}\bA^{1\rightarrow \ell -1}\bE^2_{\omega}(\bx))\\
    =&-\frac{1}{2}\sum_{\omega}\trace((\bW^{\ell})^T(\bA^{\ell +1\rightarrow L}\bQ^{\ell}_{\omega})^T\bSigma_{\bx}^{-1}\bA^{\ell+1\rightarrow L}_{\omega}\bQ^{\ell}_{\omega}\bW^{\ell}\bQ^{\ell-1}_{\omega}\bA^{1\rightarrow \ell -1}\bE^2_{\omega}(\bx)(\bQ^{\ell-1}_{\omega}\bA^{1\rightarrow \ell-1}_{\omega})^T)\\
    \implies \partial D=& -\sum_{\omega} \bQ^{\ell}_{\omega}\bA^{L\rightarrow \ell +1}\bSigma_{\bx}^{-1}\bA^{\ell+1\rightarrow L}_{\omega}\bQ^{\ell}_{\omega}\bW^{\ell}\bQ^{\ell-1}_{\omega}\bA^{1\rightarrow \ell -1}\bE^2_{\omega}(\bx)\bA^{\ell-1\rightarrow 1}_{\omega}\bQ^{\ell-1}_{\omega}
\end{align*}

\begin{align*}
     E=&-\sum_{\omega}\left( \bA^{\ell+1\rightarrow L}_{\omega}\bQ^{\ell}_{\omega}\bW^{\ell}\bQ^{\ell-1}_{\omega}\bA^{1\rightarrow \ell -1}m^1_{\omega}(\bx)\right)^T
     \bSigma_{\bx}^{-1}
     \left(\bA^{\ell+1\rightarrow L}\bQ_{\omega}^{\ell}\bW^{\ell}\bQ_{\omega}^{\ell-1}\bb_{\omega}^{1\rightarrow \ell-1}\right)\\
    =&-\sum_{\omega}\trace \left( (\bW^{\ell})^T(\bA^{\ell+1\rightarrow L}_{\omega}\bQ^{\ell}_{\omega})^T
    \bSigma_{\bx}^{-1}
    \bA^{\ell+1\rightarrow L}\bQ_{\omega}^{\ell}\bW^{\ell}\bQ_{\omega}^{\ell-1}\bb_{\omega}^{1\rightarrow \ell-1}(\bQ^{\ell-1}_{\omega}\bA^{1\rightarrow \ell -1}m^1_{\omega}(\bx))^T\right)\\
    \implies \partial E=&-\sum_{\omega}  \bQ^{\ell}_{\omega}\bA^{L\rightarrow \ell +1}\bSigma_{\bx}^{-1}\bA^{\ell+1\rightarrow L}_{\omega}\bQ^{\ell}_{\omega}\bW^{\ell}\left(\bQ^{\ell-1}\left(\bb_{\omega}^{1\rightarrow \ell-1}(m^1_{\omega}(\bx))^T\bA_{\omega}^{\ell-1\rightarrow 1}+\bA_{\omega}^{1\rightarrow \ell-1}m^1_{\omega}(\bx)(\bb_{\omega}^{1\rightarrow \ell-1})^T\right)(\bQ^{\ell-1})^T\right)
\end{align*}

we can group B,D and E together as well as A and C. Now to solve this equal $0$ we will need to consider the flatten version of $\bW^{\ell}$ which we denote by $\bw^{\ell}=\vect(\bW^{\ell})$ leading to

\begin{align*}
    \partial L=&\sum_{\omega}\bQ^{\ell}_{\omega}\bA^{L\rightarrow \ell+1}_{\omega}\bSigma_{\bx}^{-1}(\bx-\brewriteb)\left(\bQ^{\ell-1}_{\omega}(\bA^{1\rightarrow \ell -1}\be^1_{\omega}(\bx)+\bb^{1\rightarrow \ell-1}e^0_{\omega}(\bx))\right)^T\\
    &-\sum_{\omega}  \bQ^{\ell}_{\omega}\bA^{L\rightarrow \ell +1}\bSigma_{\bx}^{-1}\bA^{\ell+1\rightarrow L}_{\omega}\bQ^{\ell}_{\omega}\bW^{\ell}\bQ^{\ell-1}_{\omega}\Big(e^0_{\omega}(\bx)\bb_{\omega}^{1\rightarrow \ell-1}(\bb_{\omega}^{1\rightarrow \ell-1})^T+\bA_{\omega}^{1\rightarrow \ell-1}\bE^2_{\omega}(\bx)\bA_{\omega}^{\ell-1\rightarrow 1}\\
    &\hspace{4cm}+\bb_{\omega}^{1\rightarrow \ell-1}(m^1_{\omega}(\bx))^T\bA_{\omega}^{\ell-1\rightarrow 1}+\bA_{\omega}^{1\rightarrow \ell-1}m^1_{\omega}(\bx)(\bb_{\omega}^{1\rightarrow \ell-1})^T\Big)\bQ^{\ell-1}_{\omega}\\
    =&\sum_{\omega}P_{\omega}(\bx)^{\ell}-U^{\ell}_{\omega}\bW^{\ell}V^{\ell}_{\omega}(\bx)\\
    \implies &(\sum_{\bx}\sum_{\omega}U^{\ell}_{\omega}\otimes (V^{\ell}_{\omega}(\bx))^T )\vect(\bW^{\ell})=\sum_{\bx}\sum_{\omega}\vect(P^{\ell}_{\omega}(\bx)\\
    \implies \vect(\bW^{\ell})^* =& (\sum_{\bx}\sum_{\omega}U^{\ell}_{\omega}\otimes (V^{\ell}_{\omega}(\bx))^T )^{-1}(\sum_{\bx}\sum_{\omega}\vect(P^{\ell}_{\omega}(\bx))
\end{align*}

\section{Regularization}
\label{appendix:regularization}

We propose in this section a brief discussion on the impact of using a probabilistic prior on the weights of the GDN. In particular, it is clear that imposing a Gaussian prior with zero mean and isotropic covariance on the weights falls back in the log likelihood to impose a $l2$ regularization of the weights with parameter based on the covariance of the prior. If the prior is a Laplace distribution, the log-likelihood will turn the prior into an $l1$ regularization of the weights, again with regularization coefficient based on the prior covariance. Finally, in the case of uniform prior with finite support, the log likelihood will be equivalent to a weight clipping, a standard technique employed in DNs where the weights can not take values outside of a predefined range.

\section{Computational Complexity}
\label{appendix:complexity}

The computational complexity of the method increases drastically with the latent space dimension, and the number of regions, and the number of faces per regions. Those last quantities are directly tied into the complexity (depth and width) of the DGNs. This complexity bottleneck comes from the need to search for all regions, and the need to decompose each region into simplices. As such, the EM learning is not yet suitable for large scale application, however based on the obtained analytical forms, it is possible to derive an approximation of the true form that would be more tractable while providing approximation error bounds as opposed to current methods.

\section{Additional Experiments}
\label{appendix:more_XP}

In this section we propose to complement the toy circle experiment from the main paper first we an additional $2d$ case and then with the MNIST dataset.

{\bf Wave}

We propose here a simple example where the read data is as follows:
\begin{figure}[H]
    \centering
    \includegraphics[width=0.2\linewidth]{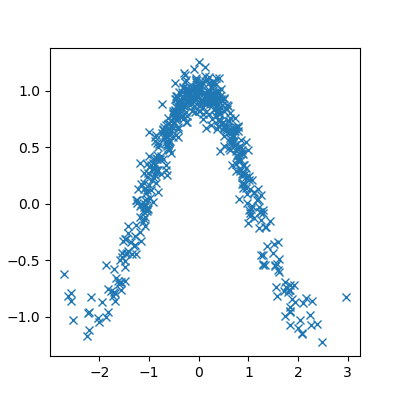}
    \caption{sample of noise data for the wave dataset}
\end{figure}

We train on this dataset the EM and VAE based learning with various learning rates and depict below the evolution of the NLL for all models, we also depict the samples after learning.

\begin{figure}
    \centering
    \includegraphics[width=0.8\linewidth]{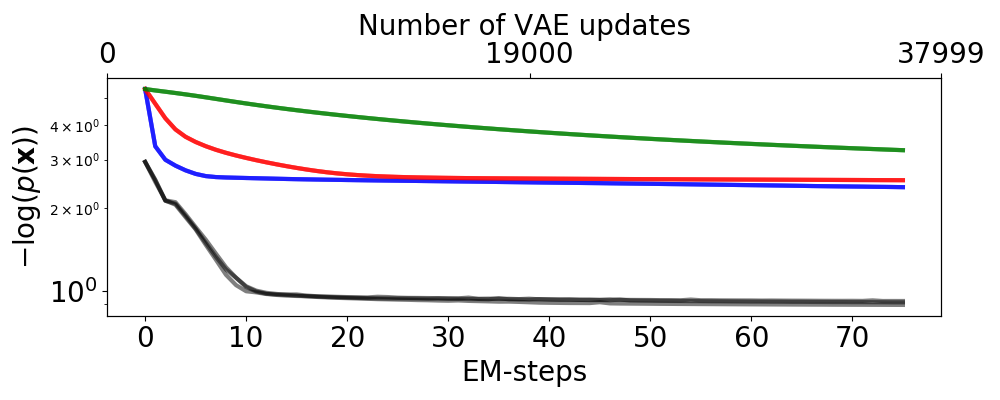}
    \caption{Depiction of the evolution of the NLL during training for the EM and VAE algorithms, we can see that despite the high number of training steps, VAEs are not yet able to correctly approximate the data distribution as opposed to EM training which benefits from much faster convergence. We also see how the VAEs tend to have a large KL divergence between the true posterior and the variational estimate due to this gap, we depict below samples from those models.}
\end{figure}

\begin{figure}
    \centering
    \includegraphics[width=0.8\linewidth]{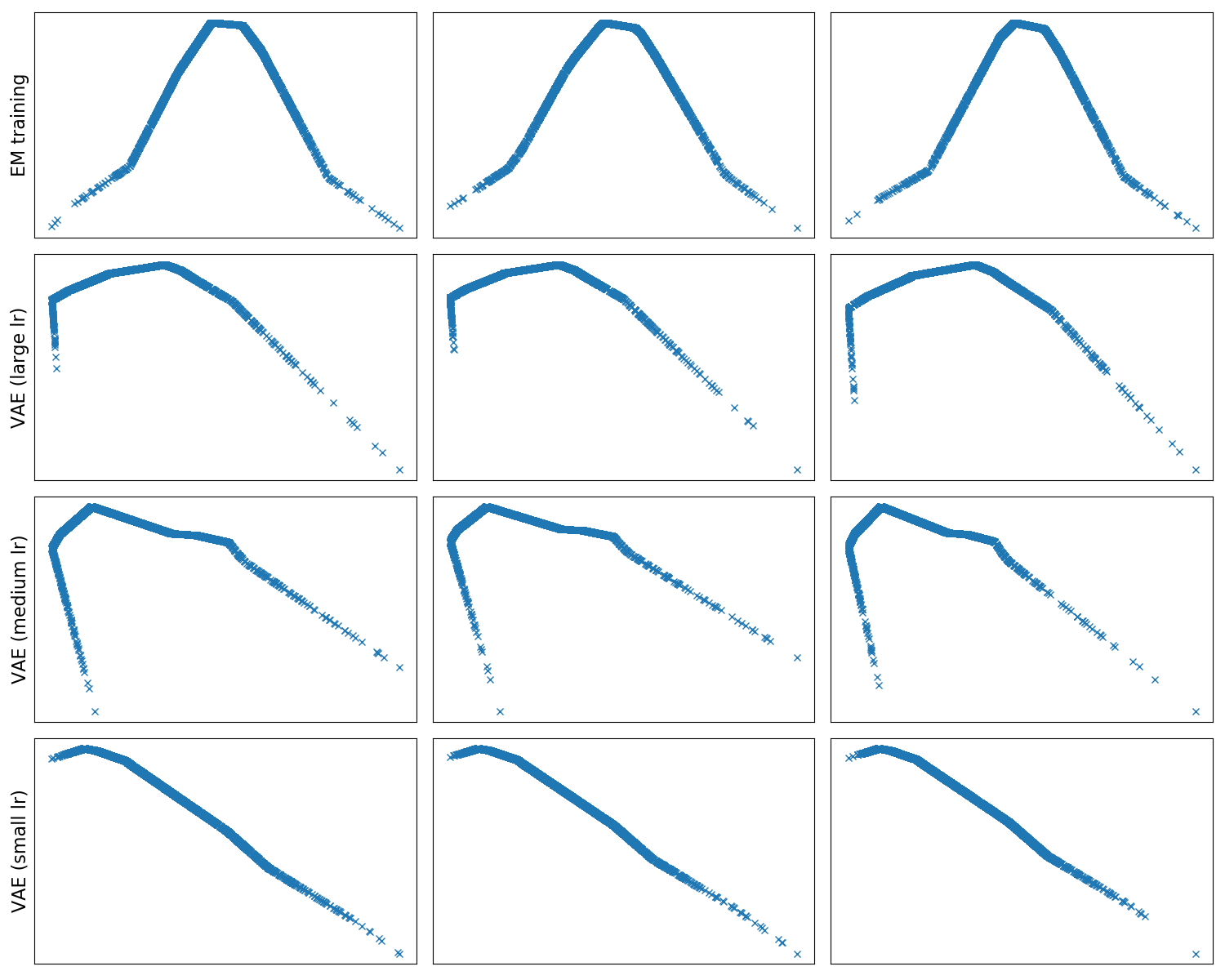}
    \caption{Samples from the various models trained on the wave dataset. We can see on {\bf top} the result of EM training where each column represents a different run, the remaining three rows correspond to the VAE training. Again, EM demonstrates much faster convergence, for VAE to reach the actual data distribution, much more updates are needed.}
\end{figure}

{\bf MNIST}
We now employ MNIST which consists of images of digits, and select the $4$ class. Note that due to complexity overhead we maintain a univariate latent space of the GDN and employ a three layer DGN with 8 and 16 hidden units. We provide first the evolution of the NLL through learning for all the training methods and then sample images from the trained DGNs demonstrating how for small DGNs EM learning is able to learn a better data distribution and thus generated realistic samples as opposed to VAEs which need much longer training steps.

\begin{figure}[h]
    \centering
    \includegraphics[width=\linewidth]{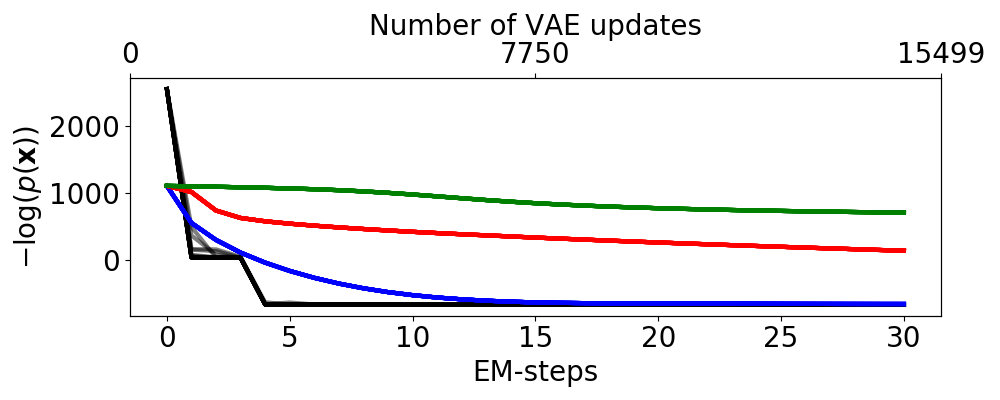}
    \caption{Evolution of the true data negative log-likelihood (in semilogy-y plot on MNIST (class $4$) for EM and VAE training for a small DGN as described above. The experiments are repeated multiple times, we can see how the learning rate is clearly impacting the learning significantly despite the use of Adam, and that even with the large learning rate, the EM learning is able to reach lower NLL, in fact the quality of the generated samples of the EM modes is much higher as shows below.}
    \label{fig:my_label}
\end{figure}

\begin{figure}[h]
    \centering
    Expectation-Maximization training\\
    \includegraphics[width=0.68\linewidth]{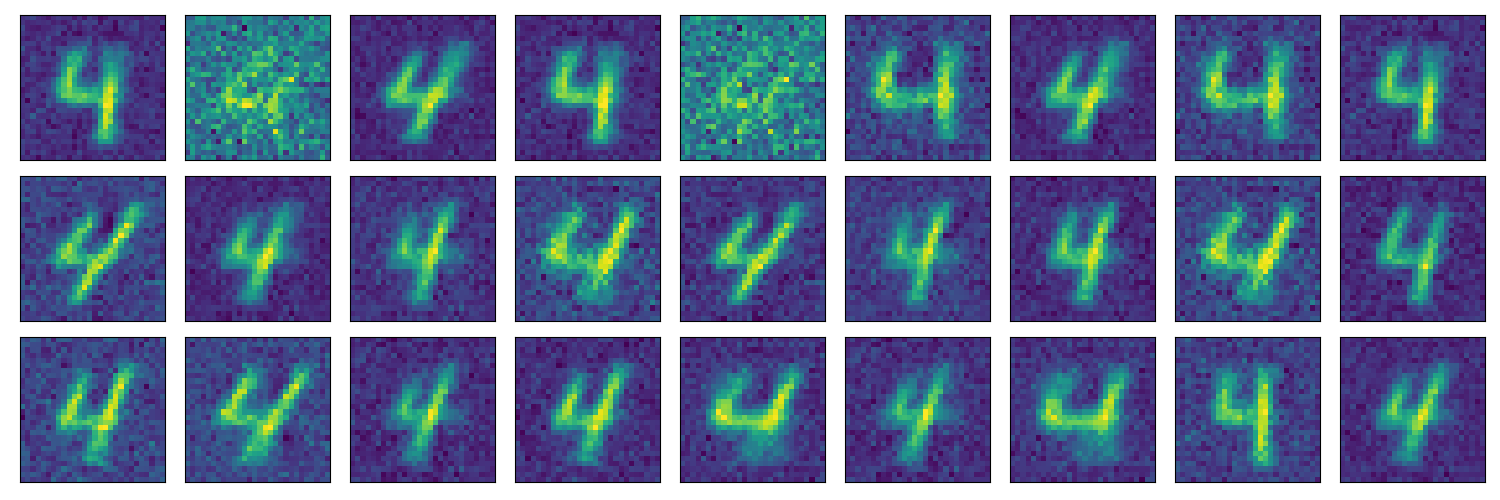}\\
    VAE training (large learning rate)\\
    \includegraphics[width=0.68\linewidth]{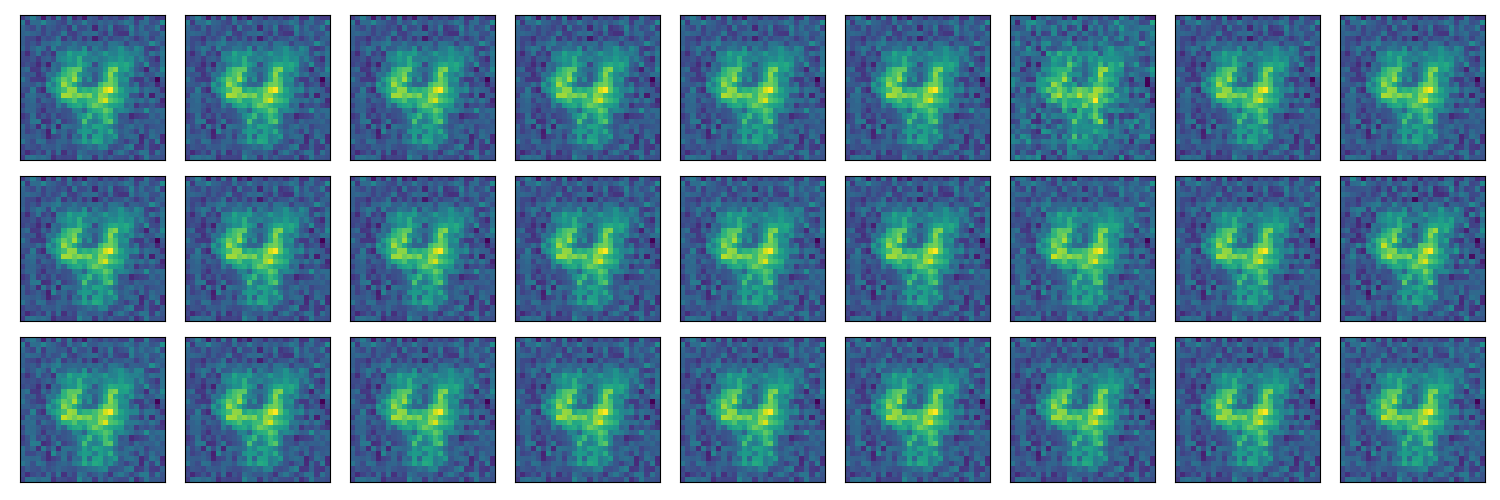}\\
    VAE training (medium learning rate)\\
    \includegraphics[width=0.68\linewidth]{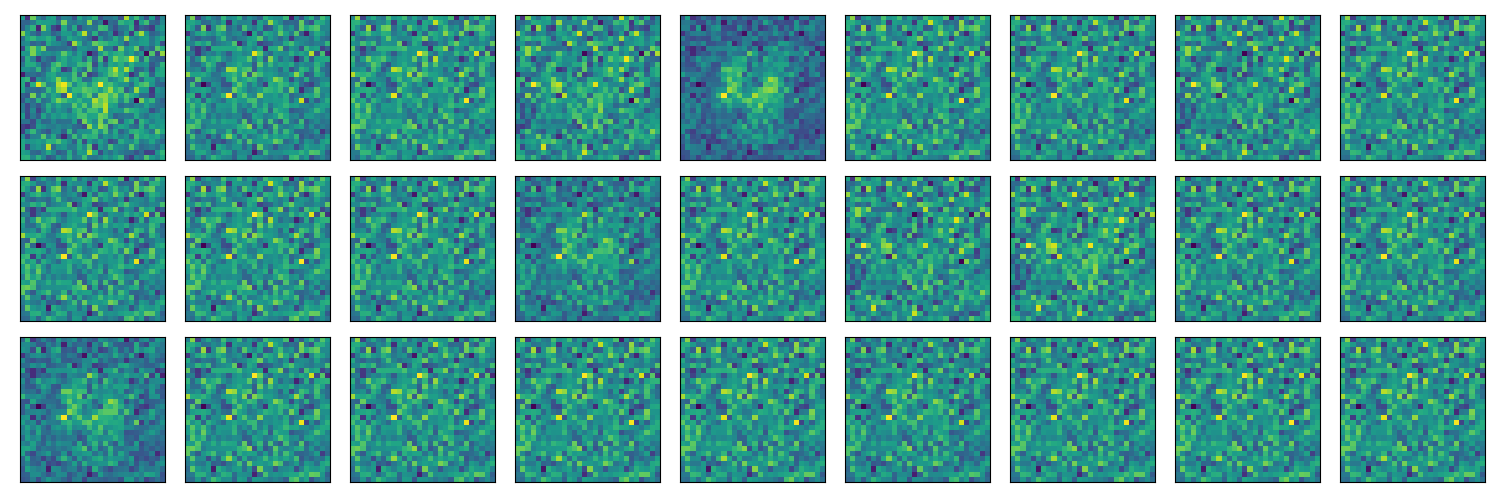}\\
    VAE training (small learning rate)\\
    \includegraphics[width=0.68\linewidth]{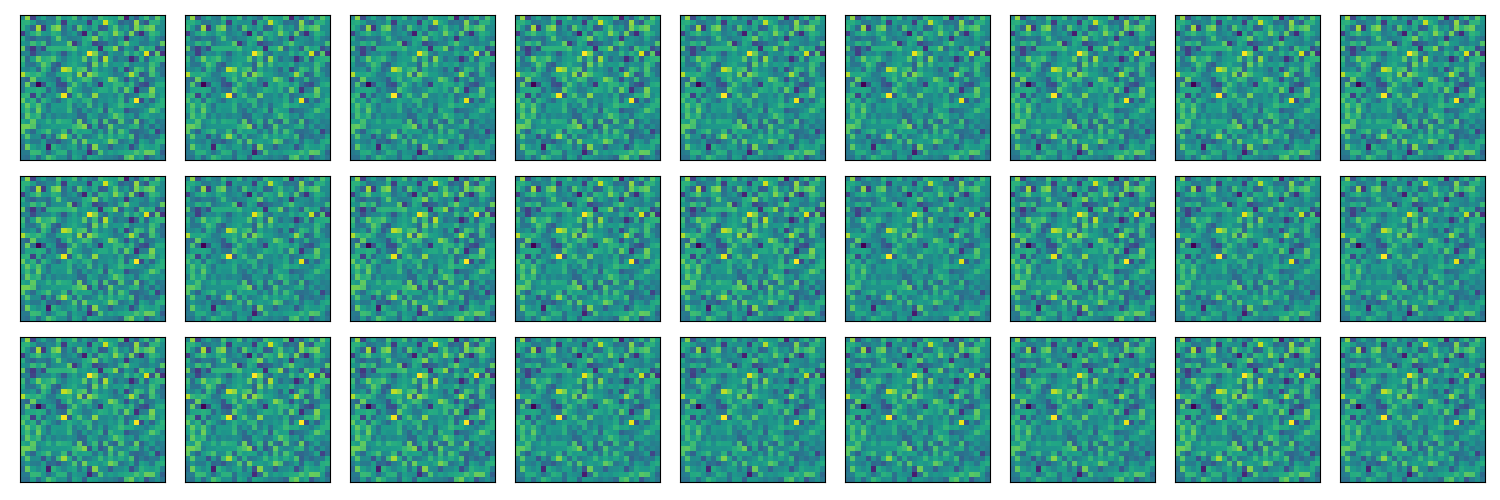}
    \caption{{\small Random samples from trained DGNs with EM or VAEs on a MNIST experiment (with digit $4$). We see the ability of EM training to produce realistic and diversified samples despite using a latent space dimension of $1$ and a small generative network.}}
    \label{fig:my_label}
\end{figure}

\end{document}